\documentclass{article}
\PassOptionsToPackage{numbers,sort&compress}{natbib}
\usepackage[final,main]{neurips_2026}
\usepackage{natbib}
\usepackage[utf8]{inputenc}
\usepackage[T1]{fontenc}
\usepackage{amsmath,amssymb,amsfonts,amsthm}
\usepackage{mathtools}
\usepackage{algorithm}
\usepackage{algorithmic}
\usepackage{graphicx}
\usepackage{xcolor}
\usepackage{booktabs}
\usepackage{hyperref}
\usepackage{cleveref}
\usepackage{float}
\usepackage{placeins}
\usepackage{bm}
\usepackage{tikz}
\usepackage{subcaption}
\usepackage{hyperref}
\usepackage{longtable}
\usepackage{etoolbox}
\makeatletter
\patchcmd{\LT@output}{\vss}{\vfil}{}{\PackageWarning{paper}{Longtable output patch not applied}}
\patchcmd{\LT@output}{\vss}{\vfil}{}{\PackageWarning{paper}{Longtable output patch not applied}}
\makeatother
\usepackage{array}

\DeclareTextFontCommand{\textsc}{\upshape\scshape}

\usepackage{titletoc}
\newcommand\DoToC{%
  \startcontents
  \printcontents{}{1}{\textbf{Table of Contents}\vskip3pt\hrule\vskip5pt}
  \vskip3pt\hrule\vskip5pt
}

\usetikzlibrary{arrows.meta,positioning,shapes,calc}
\hypersetup{hidelinks}
\usepackage{pgfplots}
\pgfplotsset{compat=1.18}
\usepgfplotslibrary{groupplots}
\usetikzlibrary{patterns}

\definecolor{uctc}{HTML}{8C8C8C}
\definecolor{puctc}{HTML}{B0B0B0}
\definecolor{mentsc}{HTML}{4D4D4D}
\definecolor{GBOPc}{HTML}{228B22}
\definecolor{gspc}{HTML}{0072B2}
\definecolor{gspfc}{HTML}{D55E00}

\tikzset{
  mctsuct/.style={uctc, dashed, mark=o, mark size=1.15pt, line width=0.55pt},
  poweruct/.style={puctc, densely dashed, mark=square*, mark size=1.05pt, line width=0.55pt},
  ments/.style={mentsc, dotted, mark=triangle*, mark size=1.15pt, line width=0.65pt},
  GBOP/.style={GBOPc, solid,mark=pentagon*,mark size=1.20pt, line width=1.0pt}, 
  gspower/.style={gspc, solid, mark=*, mark size=1.25pt, line width=1.05pt},
  gspowerf/.style={gspfc, solid, mark=diamond*, mark size=1.35pt, line width=1.05pt}
}

\newcommand{\BudgetNextGroupPlot}[1]{%
\nextgroupplot[
  xmode=log,
  log basis x=2,
  xtick={32,128,512,2048},
  xticklabels={32,128,512,2k},
  ylabel={Reward},
  #1]
}

\theoremstyle{plain}
\newtheorem{definition}{Definition}
\newtheorem{assumption}{Assumption}
\newtheorem{proposition}{Proposition}
\newtheorem{theorem}{Theorem}
\newtheorem{lemma}{Lemma}

\theoremstyle{remark}
\newtheorem{remark}{Remark}

\crefname{definition}{Definition}{Definitions}
\crefname{assumption}{Assumption}{Assumptions}
\crefname{proposition}{Proposition}{Propositions}
\crefname{theorem}{Theorem}{Theorems}
\crefname{lemma}{Lemma}{Lemmas}
\crefname{corollary}{Corollary}{Corollaries}
\crefname{remark}{Remark}{Remarks}

\newcommand{\EE}{\mathbb{E}}
\newcommand{\PP}{\mathbb{P}}
\newcommand{\RR}{\mathbb{R}}

\newcommand{\cG}{\mathcal{G}}
\newcommand{\cT}{\mathcal{T}}
\newcommand{\cM}{\mathcal{M}}
\newcommand{\cN}{\mathcal{N}}
\newcommand{\cE}{\mathcal{E}}
\newcommand{\cS}{\mathcal{S}}
\newcommand{\cA}{\mathcal{A}}

\newcommand{\Vhat}{\widehat{V}}
\newcommand{\Qhat}{\widehat{Q}}
\newcommand{\Vtil}{\widetilde{V}}
\newcommand{\Qtil}{\widetilde{Q}}
\newcommand{\Vstar}{V^{\star}}
\newcommand{\Qstar}{Q^{\star}}

\newcommand{\defeq}{\triangleq}
\newcommand{\conc}[2]{\xrightarrow[n\to\infty]{#1,#2}}
\DeclareMathOperator*{\argmax}{arg\,max}

\title{Graph-Based Stochastic Power-UCT: \\ Monte-Carlo Graph Search \\with Power Mean Estimation}

\author{
  Tung Tran \qquad Viet Bao Mai \qquad Hoang Ta \qquad Tuan Dam\thanks{Corresponding author.}\\[2pt]
  {\normalfont School of Information and Communications Technology}\\
  {\normalfont Hanoi University of Science and Technology}\\
  {\normalfont Hanoi, Vietnam}\\[2pt]
  {\normalfont\small\texttt{tung.td235240@sis.hust.edu.vn}\quad\texttt{bao.mv225474@sis.hust.edu.vn}}\\
  {\normalfont\small\texttt{hoang.taduy@hust.edu.vn}\quad\texttt{tuan.dam@hust.edu.vn}}
}

\date{}

\begin{document}

\maketitle

\begin{abstract}
Tree-based Monte-Carlo Tree Search (MCTS) duplicates the same state when it is
reached through different trajectories, which can waste simulations in stochastic
MDPs. We introduce Graph-Based Stochastic-Power-UCT (GS-Power-UCT),
which shares states reached at the same planning depth while keeping separate
values for states reached at different depths. This design applies to general
stochastic MDPs, including problems with cycles. We prove that for a fixed
planning horizon, the root estimate converges to the finite-horizon value at
rate $O(n^{-1/2})$ for each fixed finite power under the stated tuning, matching tree-based Stochastic-Power-UCT while reusing
samples across shared states. We also study two full-state variants:
GS-Power-UCT-F, which stores one node per physical state to increase
sample sharing but may mix values from different remaining horizons, and
GS-Power-UCT-F$^+$, which chooses the horizon as a function of the search budget.
For full merging we establish a polynomial error bound with evaluator-dependent
bias; adaptive-horizon consistency additionally requires both the bias and
the horizon-dependent estimation term to vanish. Experiments on stochastic planning benchmarks show improved sample
efficiency over tree-based and graph-based baselines.
\end{abstract}

\section{Introduction}
\label{sec:intro}

Monte-Carlo Tree Search (MCTS)~\citep{coulom2006efficient,browne2012survey} is a widely used family of online planning algorithms that combine Monte-Carlo sampling with forward tree search. The success of MCTS lies in adaptive exploration strategies inspired by the multi-armed bandit (MAB) literature, most notably the Upper Confidence Bound for Trees (UCT) algorithm~\citep{kocsis2006bandit}. Recent advances in coupling MCTS with deep learning~\citep{silver2016mastering,silver2017mastering,schrittwieser2020mastering} have enabled breakthroughs in complex decision-making problems.

Despite these successes, tree-based MCTS suffers from two fundamental limitations. First, theoretical analysis of the logarithmic exploration bonus in UCT is incomplete due to issues identified by~\citet{shah2020non}, leading to the development of Fixed-Depth-MCTS with polynomial bonuses~\citep{shah2022non}. Second, and more relevant to this work, tree-based MCTS does not identify states that are reachable via multiple trajectories. When a state $s$ can be reached via two different action sequences, it is represented as two separate nodes in the search tree, leading to redundant exploration and suboptimal use of computational budget.

\citet{leurent2020monte} address the second limitation by proposing Monte-Carlo Graph Search (MCGS), which merges identical states into a single node in a directed graph. They show that this merging can reduce the effective branching factor from $\kappa$ (tree) to $\kappa_\infty \leq \kappa$ (graph), yielding improved regret bounds. However, their theoretical analysis is restricted to \emph{deterministic} MDPs with an Optimism in the Face of Uncertainty (OFU) planning style, and the extension to stochastic MDPs is purely empirical.

Separately, the Stochastic-Power-UCT algorithm~\citep{dam2024power} addresses the first limitation by introducing power mean value estimation with polynomial exploration bonuses for stochastic MCTS, establishing $O(n^{-1/2})$ convergence rate for value estimation at the root node.

\paragraph{Contributions.}
We make the following contributions:

\begin{enumerate}
    \item We introduce \textbf{Graph-Based Stochastic Power-UCT}
    (GS-Power-UCT), which combines power mean backups, polynomial exploration
    bonuses, and depth-augmented graph search. By keying nodes as $(s,h)$, the
    search graph is a DAG by construction for any stochastic MDP, including MDPs
    with cycles.

    \item We prove that GS-Power-UCT preserves the
    $(\alpha,\beta)$-concentration (Definition~\ref{def:concentration}) guarantees of Stochastic-Power-UCT on
    depth-augmented graphs. The key step is a generalized graph Q-concentration
    lemma, Lemma~\ref{lem:graph_Q_concentration}, which handles child estimates
    whose visits are aggregated from multiple parents. This yields root expected
    error $O(n^{-1/2})$ under the square-root tuning specified below, for the truncated value
    $\Vtil(s_0,0)$; see Theorems~\ref{thm:main} and~\ref{thm:rate}.

    \item We characterize the cost of merging states across depths.
Proposition~\ref{prop:bias} shows that recursive full-state merging
can introduce irreducible cross-depth bias. For full
merging, we establish an instance-dependent polynomial root-value
error bound with explicit evaluator-dependent bias
(Theorem~\ref{thm:full_merge_rate}).
For GS-Power-UCT-F$^+$, we derive an adaptive-horizon bound
that yields convergence to $V^\star(s_0)$ when both the
horizon-dependent bias and estimation terms vanish
(Theorem~\ref{thm:adaptive}).

    \item We quantify the sample-sharing effect of depth-augmented graph search.
    For a fixed collection of simulated trajectories, the graph representation is
    the quotient of the unrolled tree representation obtained by identifying equal
    \((s,h)\) pairs. Theorem~\ref{thm:graph_nonworsening} gives a deterministic
    sample-sharing identity and an exact ratio of average visit counts
    between the two representations.
    This is a representation-level comparison, not an algorithm-level dominance
    claim over an independently run tree MCTS algorithm.
\end{enumerate}
\section{Preliminaries}
\label{sec:prelim}

\subsection{Markov Decision Process}

We consider a discrete-time discounted MDP $\cM = \langle \cS, \cA, R, P, \gamma \rangle$, where $\cS$ is the state space and $\cA_s\subseteq\cA$ is the nonempty set of admissible actions at each nonterminal state $s$, with $K\defeq\max_s|\cA_s|<\infty$. A simulator response at $(s,a)$ consists of a successor drawn from $P(\cdot\mid s,a)$ and a reward in $[0,R_{\max}]$; $R(s,a,s')$ denotes its conditional mean given $s'$. The discount factor is $\gamma\in[0,1)$. Physical terminal states have zero remaining reward and value, and the evaluator returns zero there. Bellman formulas below apply at nonterminal states; at physical terminals set $(\mathcal BV)(s)=V^\star(s)=V_0(s)=0$. The optimal value function satisfies the Bellman optimality equation:
\begin{equation}
    \Vstar(s) = \max_{a \in \cA_s} \Qstar(s,a), \quad \Qstar(s,a) = \sum_{s'} P(s'|s,a)\left[R(s,a,s') + \gamma \Vstar(s')\right].
\end{equation}

\subsection{MCTS with Power Mean Value Backup}

Given a planning horizon $H$ and a playout policy $\pi_0$ with value $V_0$, define inductively $\Vtil(s_H) = V_0(s_H)$ and for $h \leq H-1$:
\begin{equation}\label{eq:truncated_bellman}
    \Qtil(s_h, a) = r(s_h, a) + \gamma \sum_{s_{h+1}} P(s_{h+1}|s_h, a)\Vtil(s_{h+1}), \quad \Vtil(s_h) = \max_{a\in\cA_{s_h}} \Qtil(s_h, a),
\end{equation}
where $r(s_h, a)$ is the expected immediate reward; the truncation error satisfies $|\Qstar(s_0, a) - \Qtil(s_0, a)| \leq \gamma^H \|\Vstar - V_0\|_\infty$. Equivalently, with the Bellman optimality operator $(\mathcal B V)(s) \defeq \max_{a\in\cA_s} \sum_{s'} P(s'|s,a)[R(s,a,s')+\gamma V(s')]$, $\Vstar$ is the unique fixed point of $\mathcal B$ and the depth-$h$ truncated value is $\Vtil(s,h) = (\mathcal B^{H-h}V_0)(s)$. Writing $t$ for the local visit count of an internal node $s_h$, its \emph{power mean value estimate} is
\begin{equation}\label{eq:power_mean}
    \Vhat_t(s_h) = \left(\sum_{a \in \cA_{s_h}} \frac{T_{s_h,a}(t)}{t}\left(\Qhat_{T_{s_h,a}(t)}(s_h, a)\right)^p\right)^{1/p},
\end{equation}
where $p \in [1,+\infty)$ and $T_{s_h,a}(t)$ is the number of visits to $(s_h, a)$.
\subsection{Concentration Framework}

Following~\citet{shah2022non} and~\citet{dam2024power}, we use the following notion of polynomial concentration.

\begin{definition}[$(\alpha,\beta)$-concentration]\label{def:concentration}
    A sequence of estimators $(\Vhat_n)_{n \geq 1}$ concentrates at rate $(\alpha,\beta)$ towards some limit $V$ if there exists a constant $c > 0$ such that:
    \begin{equation}
        \forall n \geq 1,  \forall \varepsilon > 0: \quad \PP\left(|\Vhat_n - V| > \varepsilon\right) \leq c \, n^{-\alpha} \varepsilon^{-\beta}.
    \end{equation}
    We write $\Vhat_n \conc{\alpha}{\beta} V$.
\end{definition}

\section{Graph-Based Stochastic Power-UCT}
\label{sec:algorithm}

The efficiency of graph-based planning in MDPs hinges on the mechanism used to merge trajectories. We present GS-Power-UCT, an algorithm designed to balance \textbf{estimation accuracy} (preserving value consistency) with \textbf{sample efficiency} (maximizing state reuse). We formalize this balance through a state-mapping function $\phi(s, h)$ that defines how physical states $s$ and trajectory depths $h$ are indexed in the search graph $\cG_n$, allowing us to transition between a theoretically consistent model and a computationally high-throughput variant.

\subsection{Search Graph Construction}

We maintain a graph $\cG_n = (\cN_n, \cE_n)$ where each node $v \in \cN_n$ represents an aggregate state $v = \phi(s, h)$, and the choice of $\phi$ determines the topology of the search space. The \emph{depth-augmented mapping} $\phi(s, h) = (s, h)$ distinguishes the same physical state at different depths and is the \textbf{canonical form} for finite-horizon MDPs, since the value of a state is inherently non-stationary with respect to the remaining horizon $H-h$; it ensures $\cG_n$ is a DAG by construction (Prop.~\ref{prop:dag}), avoiding the mixing of different truncated targets while retaining ordinary horizon truncation error, and merges all same-depth transpositions. The \emph{full-state mapping} $\phi(s, h) = s$ instead merges all visits to a physical state regardless of depth, drastically reducing the graph size and maximizing information sharing across stages of the trajectory; the effect of this additional sharing depends on the evaluator and the encountered depths. 

\subsection{Analyzing the Merging Trade-off}
\label{sec:bias_analysis}

The transition from depth-augmented to full-state merging involves a trade-off: the former preserves depth-specific truncated targets, whereas the latter shares estimates across potentially different targets. We quantify this target discrepancy by a \textbf{cross-depth gap} $\Delta_{\textsc{cross}}$. For each physical state $s$, let $\mathcal D(s) = \{h : s \text{ is visited at depth } h \text{ during planning}\}$, with $h_{\min}(s) = \min \mathcal D(s)$ and $h_{\max}(s) = \max \mathcal D(s)$. The cross-depth gap at $s$ is
\begin{equation}
\label{eq:cross_depth_gap_}
    \delta(s) = \max_{h_1,h_2\in\mathcal D(s)} \left|\Vtil(s,h_1)-\Vtil(s,h_2)\right|.
\end{equation}
Since $\Vtil(s,h)=(\mathcal B^{H-h}V_0)(s)$ and $\mathcal B$ is a $\gamma$-contraction in $\|\cdot\|_\infty$,
\begin{equation}
\label{eq:cross_depth_gap_bound}
\begin{aligned}
    \delta(s)
    &\le \frac{\gamma^{H-h_{\max}(s)}-\gamma^{H-h_{\min}(s)}}{1-\gamma}
          \|\mathcal B V_0-V_0\|_\infty\\
    &\le \frac{(1+\gamma)\bigl(\gamma^{H-h_{\max}(s)}
                -\gamma^{H-h_{\min}(s)}\bigr)}{1-\gamma}
          \|\Vstar-V_0\|_\infty,
\end{aligned}
\end{equation}
using $\|\mathcal B V_0 - V_0\|_\infty \le (1+\gamma)\|\Vstar - V_0\|_\infty$ for the second form. The global cross-depth gap is $\Delta_{\textsc{cross}} = \max_{s\in\cG}\delta(s)$. This discrepancy is small when the Bellman residual of the evaluator is small, relevant depths are close, or all relevant remaining horizons are large. If a state is visited at a single depth, then $\delta(s)=0$. The gap compares truncated targets; it is not itself a bound on the recursive full-merge estimator. The latter is analyzed separately in Theorem~\ref{thm:full_merge_rate}.

To resolve this, we use \textbf{depth-augmented nodes}: the graph $\cG_n = (\cN_n, \cE_n)$ has nodes $(s, h) \in \cS \times \{0, \ldots, H\}$, where $s$ is the physical state and $h$ is the trajectory depth at which the node was created. Formally, $\cN_n \subseteq \cS \times \{0, \ldots, H\}$ collects the depth-augmented nodes discovered so far, and $\cE_n$ contains the observed transitions $((s,h), a, (s', h+1))$ with $(s,h), (s',h+1) \in \cN_n$ and $a \in \cA_s$. Since every edge strictly increases the depth coordinate, $\cG_n$ is a DAG by construction. Crucially, two trajectories reaching the same physical state $s$ at the same depth $h$ \emph{are merged} --- they share the node $(s, h)$ and its value estimate $\Vhat(s,h)$ --- capturing the key benefit of graph-based planning while ensuring each node has a well-defined target value $\Vtil(s_h)$ from~\eqref{eq:truncated_bellman}.

\begin{proposition}[DAG by construction]\label{prop:dag}
The search graph $\cG_n$ maintained by GS-Power-UCT is a DAG for any MDP, including those with cyclic transitions; every edge connects $(s,h)$ to $(s',h+1)$, strictly increasing the depth coordinate, so no directed cycle can exist.
\end{proposition}

The internal nodes $\mathring{\cG}_n$ (fully expanded) and boundary nodes $\partial \cG_n$ (not yet expanded) partition $\cN_n$. For each $(s, h) \in \cN_n$ we maintain \textbf{global} visit counts $T_{s,h}(n)$ (total visits to $(s,h)$), $T_{s,h,a}(n)$ (action $a$ taken at $(s,h)$), and $T_{s,h,a}^{s'}(n)$ (transitions $((s,h),a) \to (s',h+1)$). These aggregate visits from \emph{all} parents reaching $(s,h)$ — the defining feature of global-count graph-based planning.

\begin{remark}[Same-depth transpositions]
Different action sequences of equal length can reach the same state, such as
``left then up'' and ``up then left'' in a grid. Depth augmentation merges
these visits. Revisits at different depths remain separate because their
remaining planning horizons, and hence their targets, can differ.
\end{remark}
\subsection{The Four Steps of GS-Power-UCT}

\begin{table}[t]
\caption{An admissible tuning for depth-augmented concentration. Fix any $1\le p<\infty$, $C>0$, and finite horizon $H$. The resulting bonus is $C T^{1/4}/T_a^{1/2}$.}
\label{tab:conditions}
\centering
\footnotesize
\setlength{\tabcolsep}{4pt}
\renewcommand{\arraystretch}{1.1}
\begin{tabular}{ll}
\toprule
\textbf{Parameter choice} & \textbf{Index range} \\
\midrule
$\beta_0>8$ & Root exponent \\
$\beta_{h+1}=4(\beta_h+1)$ & $h=0,\ldots,H-1$ \\
$\alpha_h=\beta_h/2$, $b_h=\beta_h/4$ & $h=0,\ldots,H$ \\
\bottomrule
\end{tabular}
\end{table}

\begin{table}[t]
\caption{Comparison of the algorithm variants. The square-root rates use Table~\ref{tab:conditions} with fixed finite $p$; $K$ and $M$ bound action and successor counts. For F$^+$, $H_n=\lceil\log(n)/(2\log(1/\gamma))\rceil$ and consistency requires $B_{H_n}\to0$ and $C_{H_n}n^{-\eta_{H_n}}\to0$ (Theorem~\ref{thm:adaptive}). Bias is measured relative to the displayed target.}
\label{tab:variants}
\centering
\footnotesize
\setlength{\tabcolsep}{4pt}
\renewcommand{\arraystretch}{1.05}
\begin{tabular}{lccccc}
\toprule
\textbf{Variant} & \textbf{Node key} & \textbf{Max nodes} & \textbf{DAG?} & \textbf{Bias} & \shortstack{\textbf{Reference target}\\\textbf{and error bound}} \\
\midrule
\shortstack{Stochastic-\\Power-UCT} & path & $\sum_{h=0}^H(KM)^h$ & trivially & $0$ & $\Vtil(s_0)$, $O(n^{-1/2})$ \\
GS-Power-UCT & $(s,h)$ & $|\cS|(H+1)$ & by constr. & $0$ & $\Vtil(s_0)$, $O(n^{-1/2})$ \\
GS-Power-UCT-F & $s$ & $|\cS|$ & sometimes & $B_H$
& $V^\star(s_0)$, $B_H+O(n^{-\eta})$ \\
GS-Power-UCT-F$^+$ & $s$ & $|\cS|$ & sometimes & $B_{H_n}$
& $V^\star(s_0)$, $B_{H_n}+C_{H_n}n^{-\eta_{H_n}}$ \\
\bottomrule
\end{tabular}
\end{table}

Each simulation $t$ of GS-Power-UCT proceeds in four steps.

\paragraph{Step 1: Selection with Lookup.} Starting from the root $(s_0, 0)$, at each internal node $(s, h) \in \mathring{\cG}_n$ select an action via the UCB-style rule
\begin{equation}\label{eq:action_selection}
    a = \argmax_{a \in \cA_s} \left\{\Qhat_{T_{s,h,a}(t)}(s,h,a) + C \cdot \frac{T_{s,h}(t)^{b_{h+1}/\beta_{h+1}}}{T_{s,h,a}(t)^{\alpha_{h+1}/\beta_{h+1}}}\right\},
\end{equation}
where $\{b_h\}$, $\{\alpha_h\}$, $\{\beta_h\}$ satisfy Table~\ref{tab:conditions} and $C > 0$ is an exploration constant. Sample $s' \sim P(\cdot|s,a)$ and perform a \textbf{depth-augmented lookup}: if $(s', h+1) \in \cG_n$ (a \emph{same-depth transposition}), continue selection from the existing node, adding the edge $((s,h),a) \to (s',h+1)$ to $\cE_n$ if new; otherwise proceed to Step~2. Selection terminates at depth $H$.

\paragraph{Step 2: Expansion.} If $(s', h+1) \notin \cG_n$, add it to $\partial\cG_n$ with $T_{s',h+1} = 0$ and append the edge $((s,h),a) \to (s',h+1)$ to $\cE_n$.

\paragraph{Step 3: Evaluation.} For a new boundary node $(s', h+1) \in \partial\cG_n$, draw one fresh rollout $Z_1(s')\sim\pi_0(s')$ with $\EE[Z_1(s')]=V_0(s')$, and set $\Vhat_1(s',h+1)=Z_1(s')$ and $T_{s',h+1}=1$. This is a single initialization sample; internal value estimates subsequently updated by the adaptive search are not treated as i.i.d.\ samples.

\paragraph{Step 4: Path-only backup.} Update the visited node-action pairs in reverse order. Writing $T \equiv T_{s_h,h,a_h}(t)$ and $T' \equiv T_{s_h,h}(t)$, and letting $r_t$ be the reward observed on the corresponding transition at simulation $t$, in reverse order:
\begin{align}
    T &\leftarrow T + 1, \quad T' \leftarrow T' + 1, \label{eq:update_Q}\\
    \Qhat_T(s_h, h, a_h) &\leftarrow \tfrac{(T-1)\Qhat_{T-1}(s_h,h,a_h) + r_t + \gamma \Vhat_{T_{s_{h+1},h+1}}(s_{h+1},h+1)}{T}, \\
    \Vhat_{T'}(s_h,h) &\leftarrow \Big(\textstyle\sum_{a \in \cA_{s_h}} \tfrac{T_{s_h,h,a}(t)}{T'}\big(\Qhat_{T_{s_h,h,a}(t)}(s_h, h, a)\big)^p\Big)^{1/p}. \label{eq:update_V}
\end{align}

After $n$ simulations, the recommended action is $\hat{a}_n = \argmax_{a\in\cA_{s_0}} \Qhat_{T_{s_0,0,a}(n)}(s_0, 0, a)$ and the root value estimate is $\Vhat_n(s_0, 0)$. The detailed algorithms are presented in Appendix \ref{sec:detailed_algorithms}.

\begin{remark}[Path-only vs.\ full graph backpropagation]
We deliberately restrict backpropagation to the current trajectory: full graph backpropagation would propagate information faster but introduces update-ordering issues and circular dependencies that break the analysis, whereas path-only updates preserve the ``one simulation, one trajectory'' structure essential for the $(\alpha,\beta)$-concentration framework.
\end{remark}
\section{Theoretical Analysis}
\label{sec:theory}

The depth-augmented graph $\cG_n$ maintained by GS-Power-UCT is a DAG by Proposition~\ref{prop:dag}, without any assumption on the MDP topology. This holds even for MDPs with arbitrary cycles or self-loops. The separate probabilistic and finiteness conditions used by the concentration results are collected in Appendix~\ref{app:assumptions}; the theoretical analysis proceeds by exploiting the DAG structure together with those conditions.

\subsection{Graph Structure and Topological Order}

Since $\cG_n$ is a DAG with depth as the natural topological coordinate, every node $(s, h)$ has a well-defined depth $h \in \{0, 1, \ldots, H\}$. For each node, define:
\begin{itemize}
    \item $\textsc{Parents}(s, h) = \{((s', h-1), a) : ((s', h-1), a, (s, h)) \in \cE_n\}$: the set of parent-action pairs,
    \item $\cN_h = \{(s, h) \in \cN_n\}$: the set of nodes at depth $h$.
\end{itemize}

Unlike a tree node, a graph node can receive visits from several parent-action pairs. For $h\ge1$, its total visit count is
\[
 T_{s,h}(n)=\sum_{((s',h-1),a)\in\textsc{Parents}(s,h)}
                 T_{s',h-1,a}^{s}(n),
\]
which aggregates contributions from all its parents.

Each node $(s, h)$ has a well-defined concentration target: the truncated value $\Vtil(s_h)$ defined by the Bellman recursion~\eqref{eq:truncated_bellman} starting from depth $h$ with remaining horizon $H - h$. This is guaranteed because all visits to $(s, h)$ come from depth $h$ --- the depth-augmentation prevents mixing estimates from different effective horizons.

\begin{definition}[Transposition ratio]\label{def:transposition_ratio}
For a depth $h$ with at least one represented node, define the \emph{transposition ratio}
\begin{equation}
    \tau_h \defeq \frac{|\cN_h|}{N_h^{\cT}},
\end{equation}
where $|\cN_h|$ is the number of unique depth-augmented nodes at depth $h$ and $N_h^{\cT}$ is the number of nodes at depth $h$ in the unrolled tree $\cT(\cG)$. We have $\tau_h \in (0, 1]$, with $\tau_h = 1$ corresponding to no transpositions (tree) and $\tau_h \ll 1$ indicating heavy same-depth state merging.
\end{definition}



\subsection{Generalized Topological Concentration Lemma}

The following lemma is the key technical innovation for extending the concentration framework from trees to graphs. It generalizes Lemma~1 of~\citet{dam2024power} to handle the case where child value estimates are updated by visits from multiple parents. For each fixed state--depth--action query, the reward--successor pairs form the fresh local stream in condition (A-Fresh) of Assumption~\ref{assump:standing}; only this local stream is i.i.d., whereas the child estimates are handled through their concentration premise.
\begin{lemma}[Q-value concentration in graph]
\label{lem:graph_Q_concentration}
Consider an internal node $s$ in $\cG$ with action $a\in\cA_s$,
and enumerate its successor support as $s_1,\ldots,s_M$, with
$p_m=P(s_m\mid s,a)>0$.
For each $m\in[M]$, suppose the globally indexed child estimates
satisfy $0\le\Vhat_{m,k}\le L$, $V_m\in[0,L]$, and the fixed-count
concentration premise (H-Q) in Appendix~\ref{app:assumptions},
with a common constant $c_V<\infty$.
Assume $\alpha>1$ and $2\alpha\le\beta$.
Under (A-Fresh), let $(X_i,S_i)_{i\ge1}$ be the i.i.d.\ local
reward--successor stream, with $0\le X_i\le R_{\max}$,
$\EE[X_i]=\mu=r(s,a)$, and $S_i\sim P(\cdot\mid s,a)$.
No independence between $X_i$ and $S_i$ within a pair is required.
Write $N_m^n=\#\{i\le n:S_i=s_m\}$, and let $t_j^m$ be the local
index of the $j$-th transition to $s_m$.
Let $T_{s_m}^{\textsc{ext}}(i)$ be the nondecreasing, possibly
adaptive number of visits to $s_m$ not generated by this fixed
parent--action pair, up to the $i$-th local update. This includes
visits through other actions at the same parent. Put
$T_{s_m}^{\textsc{eff}}(i) =N_m^i+T_{s_m}^{\textsc{ext}}(i), K_{m,j}=T_{s_m}^{\textsc{eff}}(t_j^m).
$
Each local transition completes a child visit before returning its current estimate, so $K_{m,j}$ is the global count of that returned estimate and $K_{m,j+1}>K_{m,j}$. Then the running-average estimator 
\begin{equation}\label{eq:Q_graph}
    \Qhat_n(s,a)
    =\frac1n\sum_{i=1}^n X_i
    +\frac{\gamma}{n}
      \sum_{m=1}^M\sum_{j=1}^{N_m^n}\Vhat_{m,K_{m,j}}
\end{equation}
satisfies
\[
    \Qhat_n(s,a)\conc{\alpha}{\beta}
    \mu+\gamma\sum_{m=1}^M p_mV_m=\Qtil(s,a).
\]
In fact, there is a deterministic constant
$C_Q^{\mathrm{max}}<\infty$ such that
\begin{equation}\label{eq:graph_Q_maximal}
    \PP\!\left(
      \sup_{k\ge r}|\Qhat_k(s,a)-\Qtil(s,a)|>\varepsilon
    \right)
    \le C_Q^{\mathrm{max}}r^{-\alpha}\varepsilon^{-\beta},
    \qquad r\in\mathbb N,\ r\ge1,\quad\varepsilon>0.
\end{equation}
The supremum uses the same continued local execution as the
fixed-count estimators. This stronger conclusion does not require
any strengthening of the child premise (H-Q).
An empty inner sum is interpreted as zero. The concentration
constant is independent of $n$, the deviation threshold, and
the external-visit schedule, given the stated common child
concentration constant.
\end{lemma}

\subsection{Main Result: Root Node Convergence}

We can now state the main convergence result, obtained by induction over the topological order of the DAG.
For a depth-augmented node $(s,h)$, we write
$\Vtil(s,h)$ and $\Qtil(s,h,a)$ as shorthand for the depth-$h$ truncated
Bellman targets $\Vtil(s_h)$ and $\Qtil(s_h,a)$ from
Equation~\eqref{eq:truncated_bellman}.

\begin{theorem}[Convergence rate of expected payoff]\label{thm:rate}
    Under Assumption~\ref{assump:standing}, for any fixed $1\le p<\infty$ and the tuning in Table~\ref{tab:conditions}, at the nonterminal root $(s_0,0)$,
    GS-Power-UCT satisfies
    \begin{equation}
        \EE\left[
            |\Vhat_n(s_0,0)-\Vtil(s_0,0)|
        \right]
        \leq
        O(n^{-1/2}).
    \end{equation}
\end{theorem}

\subsection{Quantifying Sample Sharing}

Depth augmentation pools visits that have the same physical state and
planning depth. The following result quantifies this pooling for a fixed
collection of trajectory records.

\begin{theorem}[Deterministic sample sharing under depth augmentation]
\label{thm:graph_nonworsening}
Fix a finite depth-augmented search graph $\cG$ whose nodes are
reachable from its root, its unrolled tree $\cT(\cG)$, and $n$
recorded simulated trajectories in $\cG$.
For each depth $h$, let $\phi_h:\cT_h\to\cN_h$ map each tree
copy to its graph node, and write $\mathcal C(v)=\phi_h^{-1}(v)$.
Let $T_v^{\cG}(n)$ and $T_u^{\cT}(n)$ count the same recorded
node visits in the two representations. Then
\begin{equation}
\label{eq:graph_tree_count_identity_short}
    T_v^{\cG}(n)=\sum_{u\in\mathcal C(v)}T_u^{\cT}(n),
    \qquad
    T_v^{\cG}(n)\ge T_u^{\cT}(n)
    \quad(u\in\mathcal C(v)).
\end{equation}
The same identity holds for recorded action selections:
\[
    T_{v,a}^{\cG}(n)
       =\sum_{u\in\mathcal C(v)}T_{u,a}^{\cT}(n).
\]
The node-count identity also holds when both representations count
only the same completed internal backups.

Let $N_h(n)$ be the total number of counted depth-$h$ visits.
Whenever $N_h(n)>0$,
\begin{equation}
\label{eq:average_sample_sharing_identity_short}
    \frac{
       |\cN_h|^{-1}\sum_{v\in\cN_h}T_v^{\cG}(n)
    }{
       |\cT_h|^{-1}\sum_{u\in\cT_h}T_u^{\cT}(n)
    }
    =\frac{|\cT_h|}{|\cN_h|}
    =\frac1{\tau_h}.
\end{equation}
Here the averages range over the specified representation nodes,
including any tree copies with zero count. The graph has no more
nodes or represented transition edges than its unrolled tree.
\end{theorem}

The factor $1/\tau_h$ is an average count-sharing identity.
The unrolled records are a second representation of the same data;
they need not be trajectories generated by a separately run tree
planner. The theorem does not compare estimator errors,
concentration constants, or the performance of independently run
search algorithms.

\subsection{Computational Complexity}

\begin{proposition}[Complexity per simulation]\label{prop:complexity}
    Each simulation of GS-Power-UCT has the following complexity:
    \begin{itemize}
        \item \textbf{Selection + Lookup}: $O(H \cdot (K + T_{\textsc{hash}}))$, where $T_{\textsc{hash}}$ is the cost of a hash table lookup. For discrete states with a hash table, $T_{\textsc{hash}} = O(1)$ amortized.
        \item \textbf{Expansion}: $O(K + T_{\textsc{hash}})$ for creating a new node and adding edges.
        \item \textbf{Evaluation}: $T_{\mathrm{eval}}$ for one evaluator call.
        \item \textbf{Backpropagation}: $O(H \cdot K)$ for updating Q-values and power means along the path.
    \end{itemize}
    Total per simulation: $O(H(K+T_{\textsc{hash}})+T_{\mathrm{eval}})$. For an $O(H)$-cost evaluator, this is $O(H(K+T_{\textsc{hash}}))$, matching the tree implementation up to lookup costs.
\end{proposition}

\subsection{Finite-Time Guarantees for Full State Merging}
\label{sec:theory_full_merge}

Fix a nonterminal root $s_0$ and horizon $H\ge1$.
Let $\mathcal S_H$ be the finite set of states reachable within
$H$ transitions and $\mathcal D_H$ its nonterminal states
reachable before depth $H$, each with a nonempty finite action set.
Define
\[
\begin{aligned}
    \varepsilon_H
    &=\max_{s\in\mathcal S_H}|V_0(s)-V^\star(s)|,
    \qquad d_0=\gamma\varepsilon_H,\\
    \Delta(s,a)&=V^\star(s)-Q^\star(s,a),\qquad
    d_H=\max_{\substack{s\in\mathcal D_H,\ a\in\mathcal A_s\\
                         \Delta(s,a)\le2d_0}}\Delta(s,a).
\end{aligned}
\]
The maximum defining $d_H$ includes optimal actions. Set
\begin{equation}
\label{eq:full_merge_finite_bias}
    B_H=\max\left\{d_0+d_H,\frac{d_H}{1-\gamma}\right\}.
\end{equation}

\begin{theorem}[Finite-time bound for full merging]
\label{thm:full_merge_rate}
Let $0<\gamma<1$ and $1\le p<\infty$.
Run Algorithm~\ref{alg:full_merge} with rewards in $[0,R_{\max}]$
and evaluator samples in $[0,L]$, where $L=R_{\max}/(1-\gamma)$.
Conditional on the pre-draw history, simulator calls follow
the model law and evaluator samples have mean $V_0(s)$.

Use bonuses $C_hT_s^{u_h}/T_{s,a}^{v_h}$ with
$C_h>0$ and $0<u_h<v_h$. Let $K,M$ denote the maximum
action and successor counts, and $\rho=\max_{h<H}u_h/v_h$.
For any $\eta$ satisfying
\begin{equation}
\label{eq:full_merge_eta_conditions}
\begin{aligned}
    0<\eta<\min\left\{\frac12,\frac1p,1-\rho\right\},
    g_\eta\defeq\frac{\gamma(KM)^\eta}{1-\eta}<1,
\end{aligned}
\end{equation}
there exists an instance-dependent constant $C_H<\infty$,
independent of $n$, such that
\begin{equation}
\label{eq:full_merge_root_rate}
    \EE|\Vhat_n^H(s_0)-V^\star(s_0)|
    \le B_H+C_Hn^{-\eta},\qquad n\ge1.
\end{equation}
An admissible choice is
\begin{equation}
\label{eq:full_merge_explicit_eta}
    \eta=\frac12\min\left\{
        \frac12,\frac1p,1-\rho,
        \frac{\log(1/\gamma)}{\log(KM)+2}
    \right\}.
\end{equation}
If $V_0=V^\star$ on $\mathcal S_H$, then $B_H=0$.
\end{theorem}

Here $u_h=b_{h+1}/\beta_{h+1}$ and
$v_h=\alpha_{h+1}/\beta_{h+1}$; the usual bonus gives
$\rho=1/2$. The guarantee concerns root-value error at fixed
$H$, with constants that may depend on action-gap margins.
\subsection{Depth-Independent Exploration and Adaptive Horizons}
\label{sec:no_bias}

GS-Power-UCT-F$^+$ uses Algorithm~\ref{alg:full_merge}
with the depth-independent bonus
\begin{equation}
\label{eq:uniform_bonus}
    B(n,s,a)=C\frac{T_s(n)^{1/4}}{T_{s,a}(n)^{1/2}},
\end{equation}
giving priority to zero-count actions. Both counts refer to
completed internal backups.

For budget $n\ge2$, initialize a new search with horizon
\begin{equation}
\label{eq:adaptive_H}
    H_n=
    \left\lceil\frac{\log n}{2\log(1/\gamma)}\right\rceil,
\end{equation}
held fixed throughout its $n$ simulations.
This gives $\gamma^{H_n}\le n^{-1/2}$ for the ordinary
truncation factor; the full-merge error is controlled below.
For each fixed horizon $H$, choose an exponent $\eta_H>0$
satisfying \eqref{eq:full_merge_eta_conditions}.
The subscript records its possible dependence on $H$;
$\eta_{H_n}$ denotes this exponent at the chosen horizon $H_n$.
In the uniform-rate statement below, $\eta_0>0$ is a constant
independent of $H$ and $n$ such that $\eta_H\ge\eta_0$
for every considered horizon.
\begin{theorem}[Conditional adaptive-horizon guarantee]
\label{thm:adaptive}
Suppose Theorem~\ref{thm:full_merge_rate} applies at every
fixed horizon $H$, with quantities $B_H,C_H,\eta_H$.
Then
\begin{equation}
\label{eq:adaptive_bound_corrected}
    \EE|\Vhat_n^{H_n}(s_0)-V^\star(s_0)|
    \le B_{H_n}+C_{H_n}n^{-\eta_{H_n}}.
\end{equation}
Consequently, convergence in expected absolute error follows if
\[
    B_{H_n}\longrightarrow0,
    \qquad
    C_{H_n}n^{-\eta_{H_n}}\longrightarrow0.
\]
If, additionally, $\eta_H\ge\eta_0>0$ and
$C_H\le C_0(1+H)^q$ uniformly in $H$, for constants
$C_0<\infty$ and $q\ge0$, then
\[
    \EE|\Vhat_n^{H_n}(s_0)-V^\star(s_0)|
    \le B_{H_n}
       +O\!\left((\log n)^q n^{-\eta_0}\right).
\]
\end{theorem}

The horizon-uniform conditions require separate justification.
For a fixed evaluator $V_0$, the sets $\mathcal S_H$ are nested, so
$\varepsilon_H$ is nondecreasing. Since $B_H\ge\gamma\varepsilon_H$,
the condition $B_{H_n}\to0$ forces $V_0=V^\star$ on all states
reachable from the root. Thus this corollary does not establish removal
of a fixed imperfect evaluator's bias merely by increasing $H$.
An improving family of evaluators $V_{0,H}$ can also be used, provided
both displayed vanishing conditions are verified for that family.
\section{Experiments}
\label{sec:experiments}

We evaluate our two graph-based variants on stochastic MDP environments and compare them against four baselines. \textbf{UCT}~\citep{kocsis2006bandit} is the classical tree-based MCTS algorithm with a logarithmic exploration bonus and arithmetic-mean backup. \textbf{MENTS}~\citep{xiao2019maximum} incorporates a maximum entropy framework with a softmax backup to facilitate entropy-regularized exploration. \textbf{Stochastic-Power-UCT}~\citep{dam2024power} (P-UCT in the figure and tables) extends the polynomial-bonus design with power mean value backup, establishing $O(n^{-1/2})$ concentration for stochastic MDPs while remaining tree-based. \textbf{GBOP}~\citep{leurent2020monte} is the graph-based baseline with an OFU planning style. Our \textbf{GS-Power-UCT} (Algorithm~\ref{alg:gs_power_uct}) combines the polynomial bonus and power mean backup with a depth-augmented graph that merges same-depth transpositions, while \textbf{GS-Power-UCT-F$^+$} relaxes the merging rule to share a single node per physical state across depths. For F$^+$, at every reported budget $n$ we set $H=H(n)$ once before the search and keep it fixed for all $n$ simulations. Comprehensive experimental details, including environment specifications and hyperparameter configurations, are provided in Appendix~\ref{sec:experiment_details}. Configurations using $p=\infty$ evaluate the max-backup implementation and are outside the finite-$p$ guarantees. The depth-augmented square-root theorem uses Table~\ref{tab:conditions}; the full-merge theorem has its own finite-$p$ rate and assumptions. 
Code is available {\color{blue!30!black}\href{https://github.com/brandon-mai/gspuct}{our GitHub repository}}.

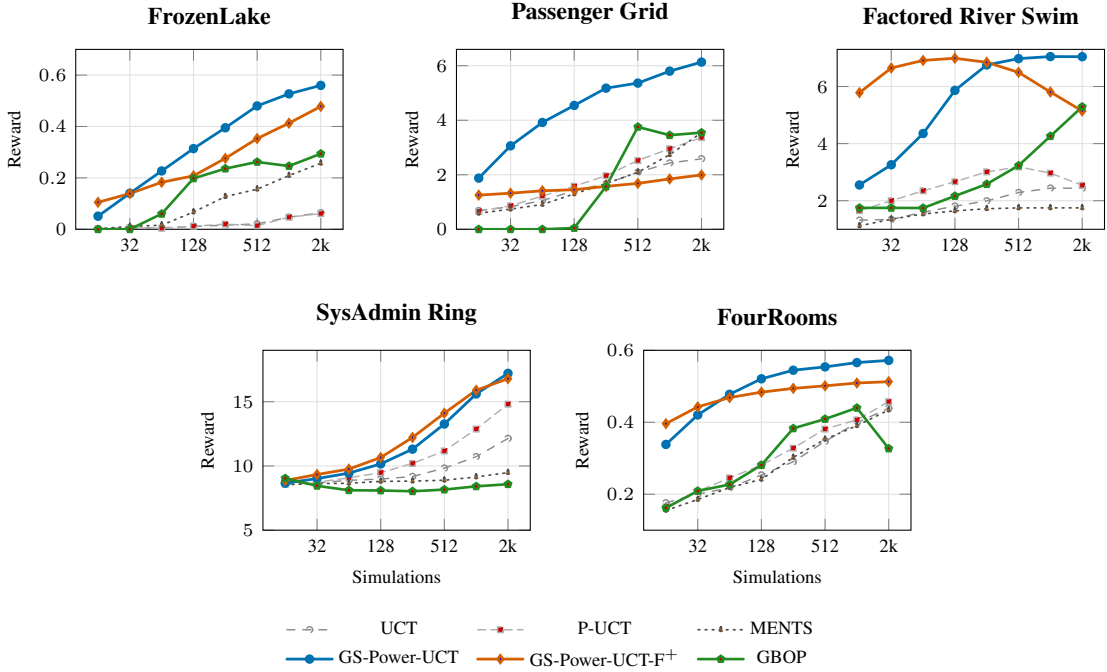
\begin{figure*}[htbp]
\centering
\begin{tikzpicture}
\begin{groupplot}[
  width=0.31\textwidth,
  height=0.24\textwidth,
  grid=major,
  major grid style={draw=black!12},
  tick label style={font=\scriptsize},
  label style={font=\scriptsize},
  title style={font=\small\bfseries},
  every axis plot/.append style={line cap=round},
  group style={
    group size=3 by 1,
    horizontal sep=1.5cm,
    vertical sep=1.3cm,
    xlabels at=edge bottom,
  },
]

\BudgetNextGroupPlot{
  title={FrozenLake},
  ymin=0, ymax=0.70,
  legend to name=planninglegend,
  legend style={font=\scriptsize, draw=none, fill=none,
      /tikz/every even column/.append style={column sep=0.45em}},
  legend columns=3,
}
\addplot+[mctsuct]  coordinates {(16,0.002) (32,0.004) (64,0.006) (128,0.011) (256,0.016) (512,0.021) (1024,0.047) (2048,0.064)};
\addplot+[poweruct] coordinates {(16,0.003) (32,0.006) (64,0.004) (128,0.012) (256,0.020) (512,0.014) (1024,0.047) (2048,0.061)};
\addplot+[ments]    coordinates {(16,0.002) (32,0.011) (64,0.016) (128,0.067) (256,0.127) (512,0.156) (1024,0.209) (2048,0.257)};
\addplot+[gspower]  coordinates {(16,0.051) (32,0.14) (64,0.227) (128,0.314) (256,0.395) (512,0.48) (1024,0.527) (2048,0.56)};
\addplot+[gspowerf] coordinates {(16,0.105) (32,0.139) (64,0.183) (128,0.208) (256,0.276) (512,0.353) (1024,0.413) (2048,0.478)};
\addplot+[GBOP]     coordinates {(16,0) (32,0) (64,0.060) (128,0.198) (256,0.236) (512,0.262) (1024,0.246) (2048,0.294)};
\legend{UCT, P-UCT, MENTS, GS-Power-UCT, GS-Power-UCT-F$^+$, GBOP}

\BudgetNextGroupPlot{title={Passenger Grid}, ymin=0, ymax=6.6}
\addplot+[mctsuct]  coordinates {(16,0.685) (32,0.828) (64,1.053) (128,1.413) (256,1.668) (512,2.093) (1024,2.432) (2048,2.586)};
\addplot+[poweruct] coordinates {(16,0.661) (32,0.861) (64,1.22) (128,1.58) (256,1.966) (512,2.524) (1024,2.955) (2048,3.366)};
\addplot+[ments]    coordinates {(16,0.587) (32,0.739) (64,0.911) (128,1.298) (256,1.679) (512,2.091) (1024,2.736) (2048,3.543)};
\addplot+[gspower]  coordinates {(16,1.877) (32,3.062) (64,3.919) (128,4.544) (256,5.178) (512,5.363) (1024,5.801) (2048,6.134)};
\addplot+[gspowerf] coordinates {(16,1.253) (32,1.324) (64,1.413) (128,1.452) (256,1.574) (512,1.684) (1024,1.847) (2048,1.993)};
\addplot+[GBOP]     coordinates {(16,0) (32,0) (64,0) (128,0.046) (256,1.570) (512,3.756) (1024,3.452) (2048,3.546)};

\BudgetNextGroupPlot{title={Factored River Swim}, ymin=1, ymax=7.3}
\addplot+[mctsuct]  coordinates {
    (16,1.31922) (32,1.34397) (64,1.59718) (128,1.80915) 
    (256,1.99507) (512,2.28615) (1024,2.43997) (2048,2.43747)
};
\addplot+[poweruct] coordinates {
    (16,1.65415) (32,1.99945) (64,2.34813) (128,2.66555) 
    (256,3.01003) (512,3.18552) (1024,2.96825) (2048,2.54082)
};
\addplot+[ments]    coordinates {
    (16,1.12058) (32,1.35172) (64,1.53537) (128,1.64488) 
    (256,1.71988) (512,1.74883) (1024,1.75) (2048,1.75)
};
\addplot+[gspower]  coordinates {
    (16,2.5517) (32,3.25923) (64,4.35333) (128,5.86532) 
    (256,6.75435) (512,6.97668) (1024,7.04837) (2048,7.04397)
};
\addplot+[gspowerf] coordinates {
    (16,5.78362) (32,6.64882) (64,6.91322) (128,6.98973) 
    (256,6.84702) (512,6.496) (1024,5.80758) (2048,5.14393)
};
\addplot+[GBOP]     coordinates {
    (16,1.75) (32,1.749) (64,1.745) (128,2.165) (256,2.583) (512,3.231) (1024,4.261) (2048,5.290)
};
\end{groupplot}
\end{tikzpicture}
\par\bigskip
\begin{tikzpicture}
\begin{groupplot}[
  width=0.31\textwidth,
  height=0.24\textwidth,
  grid=major,
  major grid style={draw=black!12},
  tick label style={font=\scriptsize},
  label style={font=\scriptsize},
  title style={font=\small\bfseries},
  every axis plot/.append style={line cap=round},
  group style={
    group size=2 by 1,
    horizontal sep=1.5cm,
    vertical sep=1.3cm,
    xlabels at=edge bottom,
  },
]

\BudgetNextGroupPlot{title={SysAdmin Ring}, xlabel={Simulations}, ymin=5, ymax=19}
\addplot+[mctsuct]  coordinates {(16,8.5736) (32,8.68655) (64,8.88255) (128,8.9796) (256,9.1801) (512,9.82805) (1024,10.7222) (2048,12.1391)};
\addplot+[poweruct] coordinates {(16,8.57455) (32,8.7314) (64,9.07755) (128,9.47875) (256,10.197) (512,11.168) (1024,12.8724) (2048,14.8099)};
\addplot+[ments]    coordinates {(16,8.5418) (32,8.63025) (64,8.65175) (128,8.79925) (256,8.82135) (512,8.88845) (1024,9.14035) (2048,9.48365)};
\addplot+[gspower]  coordinates {(16,8.66235) (32,9.01835) (64,9.4395) (128,10.1591) (256,11.3065) (512,13.264) (1024,15.6012) (2048,17.2054)};
\addplot+[gspowerf] coordinates {(16,8.86725) (32,9.33355) (64,9.7566) (128,10.6558) (256,12.2176) (512,14.1065) (1024,15.8923) (2048,16.7935)};
\addplot+[GBOP]     coordinates {(16,9.037) (32,8.452) (64,8.108) (128,8.090) (256,8.029) (512,8.159) (1024,8.414) (2048,8.580)};

\BudgetNextGroupPlot{title={FourRooms}, xlabel={Simulations}, ymin=0.1, ymax=0.60}
\addplot+[mctsuct]  coordinates {(16,0.17599) (32,0.200242) (64,0.217572) (128,0.253354) (256,0.2903) (512,0.346908) (1024,0.398532) (2048,0.43912)};
\addplot+[poweruct] coordinates {(16,0.165104) (32,0.208052) (64,0.2454) (128,0.281082) (256,0.328096) (512,0.381146) (1024,0.406628) (2048,0.457238)};
\addplot+[ments]    coordinates {(16,0.154976) (32,0.185288) (64,0.218426) (128,0.241336) (256,0.302344) (512,0.35248) (1024,0.39047) (2048,0.434064)};
\addplot+[gspower]  coordinates {(16,0.338362) (32,0.4204) (64,0.477706) (128,0.520768) (256,0.544616) (512,0.553958) (1024,0.565898) (2048,0.571962)};
\addplot+[gspowerf] coordinates {(16,0.396158) (32,0.44276) (64,0.468598) (128,0.483704) (256,0.494194) (512,0.501252) (1024,0.509114) (2048,0.512774)};
\addplot+[GBOP]     coordinates {(16,0.162) (32,0.209) (64,0.227) (128,0.281) (256,0.383) (512,0.409) (1024,0.440) (2048,0.327)};

\end{groupplot}
\end{tikzpicture}

\par\medskip
\pgfplotslegendfromname{planninglegend}
\par\medskip

\caption{
Planning performance versus simulation budget across five stochastic MDPs.
Graph-based methods are emphasized with solid curves; tree-based baselines are
shown with lighter dashed/dotted curves.
}
\label{fig:graph-planning-performance}
\end{figure*}
\paragraph{Results.}
Figure~\ref{fig:graph-planning-performance} shows that depth-augmented GS-Power-UCT is particularly effective in these stochastic domains at small and moderate simulation budgets, where same-depth transpositions can reuse observations before the budget is exhausted. GS-Power-UCT-F$^+$ is competitive in four of five domains but underperforms in \textit{Passenger Grid}, which may reflect the effect of sharing estimates across different remaining horizons. The experiment does not isolate that mechanism. Theorem~\ref{thm:full_merge_rate} bounds full-merge error through the evaluator-dependent term $B_H$; it does not predict the observed performance difference. The $1/\tau_h$ factor in Theorem~\ref{thm:graph_nonworsening} is a representation-level sample-sharing identity under the same trajectories, not a predicted performance multiplier for independently run planners.


\FloatBarrier
\section{Limitations}
\label{sec:limitations}

While \textbf{GS-Power-UCT} demonstrates robust theoretical properties, its current formulation involves several caveats that merit further discussion.

\begin{itemize}
    \item \textbf{Non-negative Reward Constraint:} The power-mean backup assumes rewards $R \in [0, R_{\max}]$. Applying GS-Power-UCT to environments with signed rewards would require affine shifts or modified backup operators to ensure value targets remain well-defined.
    \item \textbf{State sharing helps or hurts:} Sharing is useful when multiple simulated trajectories reach the same physical state at the same depth: depth-augmented GS-Power-UCT then pools their samples
without mixing finite-horizon targets. With few or no same-depth transpositions, sharing provides little reuse while retaining hash-table and graph-bookkeeping overhead. Full merging can further reduce the graph size, but its guarantee includes the evaluator-dependent error bound $B_H$, which need not vanish. The adaptive-horizon guarantee additionally requires control of horizon-dependent constants.
    \item \textbf{Computational and Memory Overhead:} Maintaining global statistics and hash-based lookups introduces additional overhead. Under depth augmentation, memory usage can scale up to $|\cS| \cdot H$ nodes, which may increase the wall-clock cost in large-scale settings despite improved sample efficiency. A detailed empirical breakdown of computational and memory overhead is presented in Appendix~\ref{sec:appendix_overhead}.
\end{itemize}

\section{Related Work}
\label{sec:related}
\paragraph{MCTS with theoretical guarantees.}
The theoretical foundations of UCT~\citep{kocsis2006bandit}
were revisited by~\citet{shah2020non,shah2022non}, who
identified a gap in the analysis of logarithmic exploration
bonuses and established polynomial convergence guarantees
for Fixed-Depth-MCTS using polynomial bonuses.
Complementary work developed a unified analysis of value
backup and exploration through power means and convex
regularization~\citep{dam2024unified}.
For stochastic MDPs, \citet{dam2024power} established
convergence guarantees for power-mean value estimation,
with subsequent work extending theoretical guarantees to
continuous stochastic planning~\citep{dam2025continuous}
and robust planning under reward and transition
uncertainty~\citep{dam2025robust}.
We extend this line with the first convergence guarantees
for graph-based MCTS in stochastic settings.
\vspace{-0.25cm}
\paragraph{Online planning and reinforcement learning.} TD($\lambda$), $n$-step TD, Dyna, and RTDP combine learning from interaction with value or model updates~\citep{sutton2018reinforcement,sutton1990dyna,barto1995rtdp}. They are complementary to the present work: GS-Power-UCT allocates a fresh online planning budget from a generative model, while a learned value function from such methods can instantiate $V_0$.
\vspace{-0.25cm}
\paragraph{Generative-model planning.}
SmoothCruiser~\citep{grill2019planning} exploits
entropy-regularized Bellman smoothness to obtain
sample-complexity guarantees, while
SecondOrderSmoothCruiser~\citep{dam2026secondorder}
develops this approach through optimal-transport
Bellman smoothing and second-order value estimation. Sparse sampling~\citep{kearns2002sparse},
OLOP~\citep{bubeck2010olop}, and
TrailBlazer~\citep{grill2016trailblazer} study planning
with a simulator, including methods that address large
or even infinite successor spaces.
Our concentration proof requires finite successor
support and focuses on sharing samples across
identical state-depth pairs in the search graph.
\vspace{-0.25cm}
\paragraph{Graph-based planning.} State merging in planning has appeared in various forms: state aggregation by partitioning~\citep{hostetler2014state}, MCGS with theoretical analysis for deterministic MDPs~\citep{leurent2020monte}, belief-state merging in POMDPs without theoretical analysis~\citep{ballesteros2013improving}, and transposition-aware/UCD search in game DAGs~\citep{childs2008transpositions,saffidine2011ucd}. Ours provides the first theoretical convergence analysis for graph-based MCTS with UCB-style exploration in stochastic MDPs.
\vspace{-0.25cm}
\paragraph{Power mean estimation.} The power mean was introduced for MCTS by~\citet{dam2019generalized} and later analyzed in~\citet{dam2024power}; closely related are MENTS~\citep{xiao2019maximum} and its R\'enyi/Tsallis variants RENTS and TENTS. Our work is the first to combine power mean estimation with graph-based planning.
\vspace{-0.25cm}
\section{Conclusion}
\label{sec:conclusion}
We introduced GS-Power-UCT, which shares same-depth state visits while preserving their truncated Bellman targets. The historical child-error analysis and a self-contained node bound establish an $O(n^{-1/2})$ fixed-horizon root rate for every fixed finite power under the stated tuning. Exact visit-count identities quantify sample reuse without claiming dominance over independently run tree planners. Full merging admits a polynomial error bound with evaluator-dependent bias; adaptive-horizon consistency requires additional vanishing conditions.

Future work includes graph-dependent constants, approximate merging for continuous states, learned evaluators, and adversarial planning.


\newpage
\bibliographystyle{abbrvnat}

\newpage
\appendix
\begin{center}
{\bf \Large{Appendix of ``Graph-Based Stochastic Power-UCT: Monte-Carlo Graph Search with Power Mean Estimation''}}
\end{center}
\DoToC
\medskip

\newpage
\section{Standing Assumptions}
\label{app:assumptions}

This appendix collects the conditions used by the concentration results. They do not restrict the topology of the MDP: cyclic transitions and self-loops remain allowed.

\begin{center}
\fbox{\begin{minipage}{0.94\linewidth}
\begin{assumption}[Standing MDP and evaluator conditions]\label{assump:standing}
\leavevmode
\begin{enumerate}
    \item \textbf{(A-Finite)} $\gamma\in[0,1)$, the planning horizon is finite, every $\cA_s$ is finite with $|\cA_s|\le K$, and every $P(\cdot\mid s,a)$ has finite support. Thus a (possibly countable) state space has a finite $H$-reachable graph. Rewards lie in $[0,R_{\max}]$; hence all rollout and backed-up values are bounded by $L\defeq R_{\max}/(1-\gamma)$.
    \item \textbf{(A-Fresh)} A generative-model call at a fixed state--depth--action query returns a fresh reward--successor pair. Successive pairs at that query are i.i.d.\ from the stationary conditional law. The reward and successor within one pair need not be independent.
    \item \textbf{(A-Rollout)} Independent calls to the terminal evaluator return $Z_j(s)\sim\pi_0(s)$ with $\EE[Z_j(s)]=V_0(s)$ and $0\le Z_j(s)\le L$. A call initializes a new boundary node; it is not an i.i.d.\ assertion about adaptively updated internal estimates. Physical-terminal evaluations are zero.
\end{enumerate}
\end{assumption}
\end{minipage}}
\end{center}

\paragraph{Local proof hypotheses (induction hypotheses).}
\textbf{(H-Q)} In Lemma~\ref{lem:graph_Q_concentration}, child
concentration is assumed only at each deterministic global visit count:
there is a deterministic constant $c_V<\infty$ such that
\[
    \PP\bigl(|\Vhat_{m,k}-V_m|>\varepsilon\bigr)
    \le c_V k^{-\alpha}\varepsilon^{-\beta},
    \qquad m\in[M],\quad k\ge1,\quad \varepsilon>0.
\]
Here $\Vhat_{m,k}$ is the snapshot after the $k$-th global child
visit in the continued graph execution, and the bound is unconditional.
Each local transition performs a
child visit before returning its estimate, so successive backups
through a fixed parent--action pair use distinct global visit counts
for each child. The proof controls the sum of these selected errors;
no concentration bound at an individual adaptively selected count,
or independence of child estimates and visit counts, is assumed.

For fixed successor-support size $M$, the graph-$Q$ constant may depend
on $M,c_V,\alpha,\beta,\gamma,R_{\max}$ and $L$; the proof requires
no additional inverse successor-probability factors. The following table records where these standing conditions and local hypotheses are used.

{\small
\setlength{\tabcolsep}{4pt}
\renewcommand{\arraystretch}{1.12}
\begin{longtable}{@{}>{\raggedright\arraybackslash}p{0.29\textwidth}>{\raggedright\arraybackslash}p{0.65\textwidth}@{}}
\caption{Assumption-discharge table for the analysis.}\label{tab:assumption_discharge}\\
\toprule
\textbf{Result or proof term} & \textbf{Condition and role} \\
\midrule
\endfirsthead
\toprule
\textbf{Result or proof term} & \textbf{Condition and role} \\
\midrule
\endhead
\bottomrule
\endfoot
\bottomrule
\endlastfoot
Leaf concentration & (A-Rollout) gives bounded i.i.d.\ terminal-evaluator calls, so Hoeffding applies to those calls. \\
\addlinespace
Graph-$Q$ local sampling term
& (A-Finite) bounds the ideal targets $X_i+\gamma V_{S_i}$;
(A-Fresh) makes their local stream i.i.d., so Hoeffding applies. \\
\addlinespace
Graph-$Q$ historical child errors
& (H-Q) gives fixed-count child concentration. Each parent backup
uses a distinct global visit count for its selected child;
the selection inequality controls their sum without independence
assumptions on child estimates or visit counts. \\
\addlinespace
Node and root concentration & The maximal conclusion of the graph-$Q$ lemma and the self-contained node bound in Lemma~\ref{lem:power_mean_node_concentration}; (H-Q) itself remains fixed-count.\\
\addlinespace
Sample-sharing theorem & This is a deterministic representation identity under the same realised trajectories; it introduces no additional stochastic assumption. \\
\addlinespace
Full merge and F$^+$
& Finite reachability, bounded conditional simulator/evaluator draws,
completed-backup statistics, and polynomial exploration give the
fixed-horizon bound. The adaptive guarantee additionally requires
its stated vanishing conditions; $H=H(n)$ is fixed throughout
each budget-$n$ search. \\
\end{longtable}
}

\newpage
\section{Notation}

{\footnotesize
\setlength{\tabcolsep}{4pt}
\renewcommand{\arraystretch}{1.12}

\begin{longtable}{@{}>{\raggedright\arraybackslash}p{0.30\textwidth}>{\raggedright\arraybackslash}p{0.68\textwidth}@{}}
\caption{Notation used throughout the paper.}
\label{tab:notation}
\\
\toprule
\textbf{Symbol} & \textbf{Meaning} \\
\midrule
\endfirsthead

\caption[]{Notation used throughout the paper (continued).}
\\
\toprule
\textbf{Symbol} & \textbf{Meaning} \\
\midrule
\endhead

\midrule
\multicolumn{2}{r}{\emph{Continued on next page}} \\
\endfoot

\bottomrule
\endlastfoot

\multicolumn{2}{@{}l}{\textit{MDP, returns, and value functions}} \\*
\(\cM=\langle\cS,\cA,R,P,\gamma\rangle\)
& Discounted MDP with state space \(\cS\), action space \(\cA\), conditional-mean reward \(R\), transition kernel \(P\), and discount \(\gamma\in[0,1)\). \\

\(\cA_s, K, R_{\max}, L\)
& Admissible actions at \(s\), \(K=\max_s|\cA_s|\), reward bound, and return bound \(L=R_{\max}/(1-\gamma)\). \\

\(s_0,s,s',a,h;\ i,j,m,t,n\)
& Root state, generic/current/successor states, action, trajectory depth, and generic sample, successor, visit, or simulation indices. \\

\(H,H(n),H_n\)
& Fixed planning horizon; adaptive horizon \(H(n)=\lceil\log n/(2\log(1/\gamma))\rceil\); and shorthand \(H_n\defeq H(n)\). Within an F$^+$ search of budget \(n\), \(H(n)\) is fixed. \\

\(R(s,a,s'),r(s,a);\ r,r_t,R_i\)
& Conditional-mean and expected immediate rewards; sampled reward, reward backed up at simulation \(t\), and the \(i\)-th sampled reward in the full-merge proof. \\

\(\pi_0,V_0,Z_j,z\)
& Terminal evaluator/policy, its value, its \(j\)-th independent output, and a realized rollout/evaluation sample. \\

\(\EE,\PP,\|\cdot\|_\infty,\mathbf 1\)
& Expectation, probability, sup norm, and indicator function. \\

\(\mathcal B\)
& Bellman optimality operator: \((\mathcal BV)(s)=\max_{a\in\cA_s}\sum_{s'}P(s'|s,a)[R(s,a,s')+\gamma V(s')]\). \\

\(\Vstar,\Qstar\)
& Infinite-horizon optimal value and action-value functions. \\

\(\Vtil(s,h),\Qtil(s,h,a),\Vtil_H(s,h)\)
& Depth-\(h\) truncated Bellman targets with terminal value \(V_0\); \(\Vtil_H\) displays the horizon explicitly. \\

\(\substack{\Vhat_t(s,h),\,\Qhat_t(s,h,a)\\[-1pt]\Vhat_t(s),\,\Qhat_t(s,a)}\)
& Empirical estimates after \(t\) visits in the depth-augmented and full-merge variants, respectively. \\

\(p\)
& Power-mean exponent. Rate theorems use finite \(p\in[1,\infty)\); \(p=\infty\) is an algorithmic implementation limit. \\

\midrule
\multicolumn{2}{@{}l}{\textit{Search graph, unrolling, and sample sharing}} \\*

\(\phi(s,h)\)
& Node-key mapping: \(\phi(s,h)=(s,h)\) for depth augmentation and \(\phi(s,h)=s\) for full merging. \\

\(\cG_n=(\cN_n,\cE_n),\ \cN_h\)
& Search graph after \(n\) simulations, its node and edge sets, and the set of discovered depth-\(h\) nodes. \\

\((s,h),\mathring{\cG}_n,\partial\cG_n,\mathcal H\)
& Depth-augmented node; internal and boundary node sets; and implementation hash map from a node key to its graph node. \\

\(\textsc{Parents}(s,h)\)
& Parent-action pairs with an edge into depth-augmented node \((s,h)\). \\

\(\cT(\cG),\cT_h;\ \phi_h:\cT_h\to\cN_h\)
& Unrolled tree of \(\cG\), its depth-\(h\) nodes, and projection of a tree copy to its graph node. \\

\(v,u,\mathcal C(v)\)
& Graph node, tree copy, and merged-copy class \(\mathcal C(v)=\phi_h^{-1}(v)\). \\

\(\tau_h,N_h^{\cT}\)
& Transposition ratio \(\tau_h=|\cN_h|/N_h^{\cT}\), where \(N_h^{\cT}=|\cT_h|\) is the number of unrolled-tree nodes at depth \(h\). \\

\(T_v^{\cG},T_u^{\cT},T_{v,a}^{\cG},T_{u,a}^{\cT},N_h(n)\)
& Graph/tree node and action visit counts in the same-trajectory comparison, and the common total number of visits reaching depth \(h\). \\

\midrule
\multicolumn{2}{@{}l}{\textit{Visit counts, bonuses, and algorithmic parameters}} \\*

\(\substack{T_{s,h}(n),\,T_{s,h,a}(n)\\[-1pt]T_{s,h,a}^{s'}(n)}\)
& Depth-augmented node visits, action selections, and observed transitions \(((s,h),a)\to(s',h+1)\). \\

\(\substack{T_s(n),\,T_{s,a}(n)\\T_{s,a}^{s'}(n),\,T_{s,a,h}(n)}\)
& Full-merge completed internal, action, successor-transition, and
depth-resolved action-backup counts; evaluator calls are excluded. \\

\(M_s(n),\widehat U_n(s)\)
& Evaluator-call count and running mean, stored separately from
internal decision statistics. \\

\(B(n,s,a),B_h(n,s,a),C\)
& Exploration bonus (with depth-indexed form when applicable) and exploration constant. For F$^+$, \(B(n,s,a)=C T_s(n)^{1/4}/T_{s,a}(n)^{1/2}\). \\

\(b_h,\alpha_h,\beta_h;\ b,\alpha_+,\beta_+\)
& Depth-indexed Power-UCT parameters and their local full-merge counterparts. \\

\(T_{\textsc{hash}}\)
& Cost of one hash-table lookup in the runtime bound. \\

\(\hat a_n\)
& Recommended root action after \(n\) simulations. \\

\midrule
\multicolumn{2}{@{}l}{\textit{Graph-\(Q\) concentration}} \\*

\(\Vhat_k\conc{\alpha}{\beta}V;\ \alpha,\beta,c,\varepsilon\)
& Fixed-count polynomial concentration:
\(\PP(|\Vhat_k-V|>\varepsilon)
\le c k^{-\alpha}\varepsilon^{-\beta}\)
for every deterministic \(k\ge1\) and \(\varepsilon>0\). \\

\(M,[M],s_m,p_m,\bm p,\widehat{\bm p}_n\)
& Successor-support size, \([M]=\{1,\ldots,M\}\),
successor states, \(p_m=P(s_m\mid s,a)\), and
true/empirical successor-probability vectors. \\

\(X_i,S_i,\mu,V_m\)
& Local reward--successor pair, reward mean
\(\mu=\EE[X_i]\), and target value of successor \(s_m\). \\

\(N_m^n,\widehat p_{m,n}\)
& Local successor count
\(N_m^n=\sum_{i=1}^n\mathbf1\{S_i=s_m\}\)
and empirical probability \(\widehat p_{m,n}=N_m^n/n\). \\

\(T_{s_m}^{\textsc{ext}}(i),T_{s_m}^{\textsc{eff}}(i)\)
& External and total global child visits up to local update \(i\),
with \(T_{s_m}^{\textsc{eff}}(i)
=N_m^i+T_{s_m}^{\textsc{ext}}(i)\).
External visits are those not generated by the fixed
parent--action pair, including visits through other actions
at the same parent. \\

\(t_j^m,K_{m,j}\)
& Local index of the \(j\)-th transition to \(s_m\), and global
count \(K_{m,j}=T_{s_m}^{\textsc{eff}}(t_j^m)\) of the child
estimate returned by that visit. \\

\(E_{m,k},I_{m,k}^{(n)},D_{m,n}\)
& Absolute error \(E_{m,k}=|\Vhat_{m,k}-V_m|\);
indicator that global count \(k\) is used among the first
\(n\) local backups; and normalized selected-error sum
\(D_{m,n}=n^{-1}\sum_{k\ge1}I_{m,k}^{(n)}E_{m,k}\). \\

\(Y_i^\star,q,B\)
& Ideal local target \(Y_i^\star=X_i+\gamma V_{S_i}\),
its mean \(q=\mu+\gamma\sum_m p_mV_m\), and a positive
upper bound \(B\) on its range. \\

\(C_{\mathrm{sel}},C_{\mathrm{hist}},\)
\(C_{\mathrm{exp}},C_Q\)
& Constants for the selection inequality, historical
child-error bound, exponential-to-polynomial conversion,
and final \(Q\)-concentration bound, respectively. \\

\midrule
\multicolumn{2}{@{}l}{\textit{Cross-depth targets}} \\*

\(\mathcal D(s),h_{\min}(s),h_{\max}(s)\)
& Visited depths of $s$, and their minimum and maximum. \\

\(\delta(s),\Delta_{\textsc{cross}}\)
& Local and global gaps between depth-specific truncated targets. \\

\(\omega_{n,h}(s,a),B_hV,\mathcal T_{F,n}^H\)
& Depth frequencies, boundary-value map, and boundary-mixture reference operator. \\

\(\bar V_n,\bar Q_n\)
& Fixed point of the reference operator and its corresponding action targets;
no estimator-convergence claim to these quantities is used. \\

\midrule
\multicolumn{2}{@{}l}{\textit{Full-merge proof auxiliaries}} \\*

\(\mathcal S_H,\mathcal D_H,\varepsilon_H\)
& States reachable within $H$ transitions, nonterminal states reachable
before depth $H$, and the maximum evaluator-mean error on $\mathcal S_H$. \\

\(d_0,d_H,B_H\)
& Discounted evaluator error, largest action gap no greater than $2d_0$,
and the finite-time bias bound in \eqref{eq:full_merge_finite_bias}. \\

\(u_h,v_h,\rho_-,\rho\)
& Exploration exponents and the minimum and maximum ratios $u_h/v_h$. \\

\(N_a(t),D,\Gamma,T_\star,R(Z)\)
& Completed local action count, bounded completion delay, positive
excess-gap margin, accuracy threshold, and bad-action weight constant. \\

\(\eta,\vartheta,g_\eta,Z,A(Z),C_H\)
& Rate and noise exponents, propagation coefficient, uniform noise
envelope, pathwise rate constant, and its expectation. \\

\end{longtable}
}

\section{Detailed Algorithms}
\label{sec:detailed_algorithms}
\subsection{Graph-Based Stochastic Power-UCT and Implementation Specification}

\begin{algorithm*}
\caption{Graph-Based Stochastic Power-UCT (GS-Power-UCT)}
\label{alg:gs_power_uct}
\begin{algorithmic}[1]
\REQUIRE Nonterminal root $s_0$, budget $n$, horizon $H$, fixed finite power $p\ge1$, evaluator $\pi_0$, exploration constant $C>0$
\STATE $\cG\gets(\{(s_0,0)\},\emptyset)$; $\mathcal H[(s_0,0)]\gets \text{node}$; $\partial\cG\gets\{(s_0,0)\}$; $\mathring{\cG}\gets\emptyset$
\STATE Initialize all statistics at $(s_0,0)$: $T_{s_0,0}=0$, $\Vhat(s_0,0)=0$, $T_{s_0,0,a}=0$, $\Qhat(s_0,0,a)=0$ for all $a\in\cA_{s_0}$
\FOR{$t=1,\ldots,n$}
    \STATE $\textsc{SimulateV}(s_0,0,t)$
\ENDFOR
\STATE \textbf{return} $\hat a_n=\argmax_{a\in\cA_{s_0}}\Qhat_{T_{s_0,0,a}}(s_0,0,a)$ and $\Vhat_n(s_0,0)$

\vspace{0.25em}
\STATE \textbf{Procedure} $\textsc{SimulateV}(s,h,t)$
\IF{$h=H$ or $s$ is terminal}
    \STATE $z\gets \pi_0(s)$; \quad $T_{s,h}\gets T_{s,h}+1$; \quad
    $\Vhat_{T_{s,h}}(s,h)\gets \Vhat_{T_{s,h}}(s,h)+\frac{z-\Vhat_{T_{s,h}}(s,h)}{T_{s,h}}$
    \STATE \textbf{return} $\Vhat_{T_{s,h}}(s,h)$
\ENDIF
\IF{$(s,h)\in\partial\cG$}
    \STATE Move $(s,h)$ from $\partial\cG$ to $\mathring{\cG}$
\ENDIF
\IF{$\exists a\in\cA_s$ with $T_{s,h,a}=0$}
    \STATE Choose such an action $a$
\ELSE
    \STATE
    $
    a\gets\argmax_{a'\in\cA_s}
    \left\{
    \Qhat_{T_{s,h,a'}}(s,h,a')
    +
    C\frac{T_{s,h}^{\,b_{h+1}/\beta_{h+1}}}
    {T_{s,h,a'}^{\,\alpha_{h+1}/\beta_{h+1}}}
    \right\}
    $
\ENDIF
\STATE Sample $s'\sim P(\cdot|s,a)$ and $r\sim R(s,a,s')$; set $v'=(s',h+1)$
\IF{$v'\in\mathcal H$}
    \STATE Add edge $((s,h),a,v')$ to $\cE$ if new \COMMENT{same-depth transposition}
    \STATE $V_{\mathrm{child}}\gets \textsc{SimulateV}(s',h+1,t)$
\ELSE
    \STATE Add $v'$ to $\cG$ and $\partial\cG$; $\mathcal H[v']\gets\text{node}$; add edge $((s,h),a,v')$
    \STATE Initialize all statistics at $v'$; \quad
    $z\gets Z_1(s')\sim\pi_0(s')$; \quad $T_{s',h+1}\gets1$; \quad $\Vhat_{T_{s',h+1}}(s',h+1)\gets z$
    \STATE $V_{\mathrm{child}}\gets \Vhat_{T_{s',h+1}}(s',h+1)$
\ENDIF
\STATE $T_{s,h,a}^{s'}\gets T_{s,h,a}^{s'}+1$; \quad
$T_{s,h,a}\gets T_{s,h,a}+1$
\STATE $Y\gets r+\gamma V_{\mathrm{child}}$; \quad
$\Qhat_{T_{s,h,a}}(s,h,a)\gets \Qhat_{T_{s,h,a}}(s,h,a)+\frac{Y-\Qhat_{T_{s,h,a}}(s,h,a)}{T_{s,h,a}}$
\STATE $T_{s,h}\gets T_{s,h}+1$
\STATE
$
\Vhat_{T_{s,h}}(s,h)\gets
\left(
\sum_{a'\in\cA_s}
\frac{T_{s,h,a'}}{T_{s,h}}
\bigl(\Qhat_{T_{s,h,a'}}(s,h,a')\bigr)^p
\right)^{1/p}
$
\STATE \textbf{return} $\Vhat_{T_{s,h}}(s,h)$
\end{algorithmic}
\end{algorithm*}
\begin{remark}[Implementation conventions]
The forced-initialization step avoids division by zero in the UCB-style rule
when $T_{s,h,a}=0$. Equivalently, unvisited actions may be assigned infinite
exploration bonus until every action has been sampled once.

When the depth-augmented lookup finds an existing node $(s',h+1)$, the algorithm
does not create a new node and does not perform a fresh evaluation rollout solely
because of the transposition. Instead, it continues selection from the existing
node and reuses the global statistics already stored at $(s',h+1)$.

Backpropagation remains path-only: only the nodes and edges traversed in the
current simulation are updated. Other parents of a transposed node are not
updated during this simulation.

Table~\ref{tab:conditions} specifies a feasible schedule for every fixed finite $p\ge1$ and every finite $H$, with
\[
\frac{\alpha_h}{\beta_h} = \frac{1}{2} \quad \text{and} \quad \frac{b_h}{\beta_h} = \frac{1}{4}, \quad \forall h \in [0, H],
\]
the root expected absolute error is $O(n^{-1/2})$. This choice gives the exploration bonus:
\[
B_h(n,s,a) = C \, \frac{n^{1/4}}{T_{s,a}(n)^{1/2}},
\]
where \(C\) denotes a constant controlling the degree of exploration.

As \(p \to \infty\), the power mean converges to the maximum operator. Therefore, in this limit, the backup reduces to
\[
\Vhat_{T_{s,h}}(s,h)\gets
\max_{a'\in\cA_s} \Qhat_{T_{s,h,a'}}(s,h,a').
\]
This defines the max-backup implementation at $p=\infty$. The finite-$p$ concentration constants need not remain bounded in this limit, so the rate theorems do not cover it.

\end{remark}
\subsection{Practical Variant: Full State Merging}
\label{sec:full_merge}

GS-Power-UCT-F keys nodes by physical state $s$, sharing statistics
across trajectory depths. The graph may contain cycles, while each
simulation remains finite because recursion stops at depth $H$.

\paragraph{Statistics and exploration.}
Each state stores completed internal counts $T_s,T_{s,a}$,
action means $\Qhat(s,a)$, and a power-mean value $\Vhat(s)$.
Evaluator calls use separate statistics $M_s,\widehat U(s)$.
Each evaluator sample has conditional mean $V_0(s)$ given the
pre-draw history, regardless of trajectory depth.

For actions with positive completed counts, GS-Power-UCT-F uses
\[
    B_h(s,a)
    =
    C\frac{T_s^{\,b_{h+1}/\beta_{h+1}}}
           {T_{s,a}^{\,\alpha_{h+1}/\beta_{h+1}}},
\]
whereas GS-Power-UCT-F$^+$ uses
$B_h(s,a)=CT_s^{1/4}/T_{s,a}^{1/2}$.
Both variants select zero-count actions first.

\begin{algorithm*}
\caption{Full-Merge Graph-Based Stochastic Power-UCT}
\label{alg:full_merge}
\small
\begin{algorithmic}[1]
\REQUIRE Root $s_0$, budget $n\ge1$, horizon $H\ge1$,
power $1\le p<\infty$, discount $0\le\gamma<1$,
evaluator $\pi_0$, and exploration bonus $B_h$

\STATE Create an empty state-keyed graph and hash table
\STATE $\textsc{Initialize}(s_0)$
\FOR{$t=1,\ldots,n$}
    \STATE $\textsc{SimulateV}(s_0,0)$
\ENDFOR
\IF{$T_{s_0}>0$}
    \STATE \textbf{return}
    $\operatorname*{argmax}_{a\in\cA_{s_0}}\Qhat(s_0,a)$
    and $\Vhat(s_0)$
\ELSE
    \STATE \textbf{return} no action and $\widehat U(s_0)$
\ENDIF

\STATE \textbf{Procedure} $\textsc{Initialize}(s)$
\STATE Insert the new state $s$ into the graph and hash table
\STATE $T_s\gets0$; $M_s\gets0$;
       $\Vhat(s)\gets0$; $\widehat U(s)\gets0$
\STATE Set $T_{s,a}\gets0$ and $\Qhat(s,a)\gets0$
       for every $a\in\cA_s$; initialize transition counts to zero

\STATE \textbf{Procedure} $\textsc{Evaluate}(s)$
\STATE Draw a fresh evaluator return $z$ at $s$ using $\pi_0$
\STATE $M_s\gets M_s+1$
\STATE $\widehat U(s)\gets
       \widehat U(s)+(z-\widehat U(s))/M_s$
\STATE \textbf{return} the scalar value $\widehat U(s)$

\STATE \textbf{Procedure} $\textsc{SimulateV}(s,h)$
\IF{$h=H$ or $s$ is terminal}
    \STATE \textbf{return} $\textsc{Evaluate}(s)$
\ENDIF
\IF{$T_{s,a}=0$ for some $a\in\cA_s$}
    \STATE Choose such an action $a$
\ELSE
    \STATE $a\gets
    \operatorname*{argmax}_{a'\in\cA_s}
    \{\Qhat(s,a')+B_h(s,a')\}$
\ENDIF

\STATE Draw a fresh reward--successor pair $(r,s')$
       from the generative model at $(s,a)$
\IF{$s'$ is absent from the hash table}
    \STATE $\textsc{Initialize}(s')$
    \STATE $W\gets\textsc{Evaluate}(s')$
\ELSE
    \STATE $W\gets\textsc{SimulateV}(s',h+1)$
\ENDIF
\STATE Add edge $(s,a,s')$ if absent

\STATE Read the current shared $T_{s,a}$ and $\Qhat(s,a)$
       after the child call
\STATE $Y\gets r+\gamma W$; \quad $k\gets T_{s,a}+1$
\STATE $\Qhat(s,a)\gets
       \Qhat(s,a)+(Y-\Qhat(s,a))/k$
\STATE $T_{s,a}\gets k$; \quad $T_s\gets T_s+1$;
       \quad $T_{s,a}^{s'}\gets T_{s,a}^{s'}+1$
\STATE $\displaystyle
    \Vhat(s)\gets
    \left(
        \sum_{a'\in\cA_s}
        \frac{T_{s,a'}}{T_s}\Qhat(s,a')^p
    \right)^{1/p}$
\STATE \textbf{return} the scalar value $\Vhat(s)$
\end{algorithmic}
\end{algorithm*}

\paragraph{Count invariants and cyclic updates.}
Only new states are initialized. Evaluator calls update
$M_s,\widehat U(s)$ without changing internal counts or estimates;
hence a newly evaluated state still has $T_s=0$.
Each completed internal backup increments its action and state
counts once, maintaining
\[
    T_s=\sum_{a\in\cA_s}T_{s,a},
    \qquad
    \Qhat(s,a)=
    \frac{1}{T_{s,a}}\sum_{j=1}^{T_{s,a}}Y_{s,a,j},
    \quad T_{s,a}>0,
\]
where $Y_{s,a,j}$ are ordered by completion.

All call variables are local, and $W$ is a scalar snapshot.
Since a recursive descendant may update the same state, the
parent reads its current shared statistics after recursion.
The subsequent backup is completed without another recursive call.
Thus every internal child return follows a distinct completed
child update, and the power-mean weights remain normalized.

For a nonterminal root, the estimate after $n$ simulations is
$\Vhat_n^H(s_0)=\Vhat(s_0)$; cyclic revisits may give
$T_{s_0}>n$. Theorem~\ref{thm:full_merge_rate} bounds its
error using the evaluator-dependent bias $B_H$.
\section{Detailed Proofs}
\label{app:proofs}

\subsection{Proof of Lemma~\ref{lem:graph_Q_concentration}}
\label{app:proof_graph_Q}

\begin{proof}
We first establish an inequality for sums evaluated at distinct,
possibly adaptive indices. This will allow us to use only the
fixed-count premise (H-Q).

\paragraph{An auxiliary selection inequality.}
Let $(E_k)_{k\ge1}$ be nonnegative random variables such that
\[
    \PP(E_k>x)\le c k^{-\alpha}x^{-\beta},
    \qquad k\ge1,\quad x>0,
\]
where $1<\alpha<\beta$.
Let $I_k\in\{0,1\}$ be arbitrary random selectors satisfying
$\sum_{k\ge1}I_k\le n$ almost surely. No independence between
the selectors and the variables is assumed. Set
\[
    \rho=\frac{\alpha}{\beta},\qquad
    \sigma=1-\frac1\beta,\qquad
    C_{\mathrm{sel}}
    =c\left[\frac{\beta}{\beta-1}
      \left(\frac1{1-\rho}
            +(\alpha-1)^{-1/\beta}\right)\right]^\beta.
\]
We claim that
\begin{equation}\label{eq:graph_Q_selection_bound}
    \PP\left(\frac1n\sum_{k\ge1}I_kE_k>x\right)
    \le C_{\mathrm{sel}}n^{-\alpha}x^{-\beta}.
\end{equation}

To prove this, fix $x>0$ and write
\[
    S=\sum_{k\ge1}I_kE_k,\qquad
    A=\{S>nx\},\qquad
    u=\PP(A),\qquad a_k=\PP(A\cap\{I_k=1\}).
\]
The case $u=0$ is immediate. Otherwise, $a_k\le u$ and
$\sum_{k\ge1}a_k\le nu$. Tail integration gives
\begin{align*}
    \EE[E_k I_k\mathbf1_A]
    &\le\int_0^\infty
       \min\{a_k,c k^{-\alpha}y^{-\beta}\}\,dy\\
    &=\frac{\beta}{\beta-1}c^{1/\beta}
      k^{-\rho}a_k^\sigma.
\end{align*}
Splitting the sum at $n$, we have
\[
    \sum_{k=1}^n k^{-\rho}a_k^\sigma
    \le u^\sigma\sum_{k=1}^n k^{-\rho}
    \le\frac{n^{1-\rho}u^\sigma}{1-\rho}.
\]
For the remaining terms, H\"older's inequality yields
\begin{align*}
    \sum_{k>n}k^{-\rho}a_k^\sigma
    &\le
    \left(\sum_{k>n}k^{-\alpha}\right)^{1/\beta}
    \left(\sum_{k>n}a_k\right)^\sigma \le(\alpha-1)^{-1/\beta}n^{1-\rho}u^\sigma.
\end{align*}
Using Tonelli's theorem for the nonnegative summands,
\begin{align*}
    nxu
    &\le\EE[S\mathbf1_A] \le\frac{\beta}{\beta-1}c^{1/\beta}
      \left(\frac1{1-\rho}+(\alpha-1)^{-1/\beta}\right)
      n^{1-\rho}u^{1-1/\beta}.
\end{align*}
Dividing by $nu^{1-1/\beta}$ and raising to the power $\beta$
proves \eqref{eq:graph_Q_selection_bound}.

\paragraph{Historical child-estimation errors.}
Fix $n\ge1$. For each successor $m$, define
\[
    E_{m,k}=|\Vhat_{m,k}-V_m|,\qquad
    I_{m,k}^{(n)}
    =\mathbf1\{k\in\{K_{m,1},\ldots,K_{m,N_m^n}\}\}.
\]
If $N_m^n=0$, the selected set is empty. Since the counts
$K_{m,j}$ are strictly increasing in $j$,
\[
    \sum_{k\ge1}I_{m,k}^{(n)}=N_m^n\le n,
    \qquad
    D_{m,n}:=\frac1n\sum_{j=1}^{N_m^n}
                    |\Vhat_{m,K_{m,j}}-V_m|
             =\frac1n\sum_{k\ge1}I_{m,k}^{(n)}E_{m,k}.
\]
The conditions $\alpha>1$ and $2\alpha\le\beta$ imply
$1<\alpha<\beta$. Therefore (H-Q) and
\eqref{eq:graph_Q_selection_bound}, with $c=c_V$, give
\[
    \PP(D_{m,n}>x)
    \le C_{\mathrm{sel}}n^{-\alpha}x^{-\beta},
    \qquad x>0.
\]
For $\gamma>0$, a union bound consequently yields
\begin{align}
    \PP\left(\gamma\sum_{m=1}^M D_{m,n}>\eta\right)
    &\le\sum_{m=1}^M
       \PP\left(D_{m,n}>\frac{\eta}{\gamma M}\right)
       \nonumber\\
    &\le C_{\mathrm{hist}}n^{-\alpha}\eta^{-\beta},
    \qquad
    C_{\mathrm{hist}}=C_{\mathrm{sel}}\gamma^\beta M^{\beta+1}.
    \label{eq:graph_Q_historical_bound}
\end{align}
For $\gamma=0$, this bound holds with $C_{\mathrm{hist}}=0$.

\paragraph{Local sampling error and conclusion.}
Set
\[
    q=\mu+\gamma\sum_{m=1}^M p_mV_m,
    \qquad Y_i^\star=X_i+\gamma V_{S_i},
\]
where $V_{S_i}=V_m$ when $S_i=s_m$.
By (A-Fresh), $(Y_i^\star)_{i\ge1}$ is i.i.d.\ with mean $q$.
Let $B>0$ be a deterministic upper bound on its range, for example
$B=1+R_{\max}+\gamma L$.
No independence between $X_i$ and $S_i$ within a pair is needed.
The estimator in \eqref{eq:Q_graph} satisfies
\begin{equation}\label{eq:graph_Q_decomposition}
    |\Qhat_n(s,a)-q|
    \le\left|\frac1n\sum_{i=1}^n(Y_i^\star-q)\right|
      +\gamma\sum_{m=1}^M D_{m,n}.
\end{equation}
Hoeffding's inequality and
\eqref{eq:graph_Q_historical_bound} now imply
\[
    \PP\bigl(|\Qhat_n(s,a)-q|>\varepsilon\bigr)
    \le 2\exp\left(-\frac{n\varepsilon^2}{2B^2}\right)
       +2^\beta C_{\mathrm{hist}}
        n^{-\alpha}\varepsilon^{-\beta}.
\]
Let
\[
    C_{\mathrm{exp}}
    =2\sup_{z>0}z^{\beta/2}
        \exp\left(-\frac{z}{2B^2}\right)<\infty.
\]
Since $2\alpha\le\beta$ and $n\ge1$,
\[
    2\exp\left(-\frac{n\varepsilon^2}{2B^2}\right)
    \le C_{\mathrm{exp}}n^{-\beta/2}\varepsilon^{-\beta}
    \le C_{\mathrm{exp}}n^{-\alpha}\varepsilon^{-\beta}.
\]
Thus, with $C_Q=C_{\mathrm{exp}}+2^\beta C_{\mathrm{hist}}$,
\[
    \PP\bigl(|\Qhat_n(s,a)-\Qtil(s,a)|>\varepsilon\bigr)
    \le C_Q n^{-\alpha}\varepsilon^{-\beta},
    \qquad n\ge1,\quad\varepsilon>0.
\]
This proves the fixed-count concentration using only fixed-count
child bounds and distinct global indices for successive child visits.

\paragraph{Maximal concentration of the parent-action estimates.}
Write
\[
    S_k=\gamma\sum_{m=1}^M
         \sum_{j=1}^{N_m^k}|\Vhat_{m,K_{m,j}}-V_m|,
    \qquad
    M_k=\sum_{i=1}^k(Y_i^\star-q).
\]
The historical snapshots in $S_k$ are the values returned by their
individual visits. Consequently, $S_k$ is nonnegative and
nondecreasing in $k$, even when later visits change the stored child
estimates. Equation~\eqref{eq:graph_Q_historical_bound} gives
\[
    \PP(S_k/k>x)\le C_{\mathrm{hist}}k^{-\alpha}x^{-\beta}.
\]
Fix an integer $r\ge1$. For $2^jr\le k<2^{j+1}r$,
\[
    \frac{S_k}{k}
    \le \frac{S_{2^{j+1}r}}{2^jr}
    =2\frac{S_{2^{j+1}r}}{2^{j+1}r}.
\]
A union bound over these dyadic blocks therefore yields
\begin{equation}\label{eq:graph_Q_history_maximal}
\begin{aligned}
    \PP\!\left(\sup_{k\ge r}\frac{S_k}{k}>x\right)
    &\le\sum_{j=0}^{\infty}
        C_{\mathrm{hist}}(2^{j+1}r)^{-\alpha}(x/2)^{-\beta}\\
    &=\frac{2^{\beta-\alpha}C_{\mathrm{hist}}}
            {1-2^{-\alpha}}\,
       r^{-\alpha}x^{-\beta}.
\end{aligned}
\end{equation}

For the centered i.i.d.\ ideal targets, the exponential
supermartingale proof of Hoeffding's inequality gives
\[
    \PP\!\left(\max_{1\le k\le N}|M_k|>u\right)
    \le2\exp\!\left(-\frac{2u^2}{NB^2}\right),
    \qquad N\ge1,\quad u>0,
\]
where $B$ is the deterministic range bound already defined above.
Applying this inequality with $N=2^{j+1}r$ and $u=2^jrx$
gives
\[
    \PP\!\left(\sup_{k\ge r}\frac{|M_k|}{k}>x\right)
    \le\sum_{j=0}^{\infty}
         2\exp\!\left(-\frac{2^jrx^2}{B^2}\right).
\]
Let
\[
    C_{\mathrm{mg}}
       =2\sup_{z>0}z^{\beta/2}\exp(-z/B^2)<\infty.
\]
Since $2\alpha\le\beta$ and $r\ge1$, the last display is at most
\begin{equation}\label{eq:graph_Q_sampling_maximal}
    \frac{C_{\mathrm{mg}}}{1-2^{-\beta/2}}
      r^{-\beta/2}x^{-\beta}
    \le
    \frac{C_{\mathrm{mg}}}{1-2^{-\beta/2}}
      r^{-\alpha}x^{-\beta}.
\end{equation}
Finally, the decomposition~\eqref{eq:graph_Q_decomposition}
holds simultaneously for all local counts. Combining
\eqref{eq:graph_Q_history_maximal} and
\eqref{eq:graph_Q_sampling_maximal}, each at $x=\varepsilon/2$,
proves~\eqref{eq:graph_Q_maximal} with
\[
 C_Q^{\mathrm{max}}
 =2^\beta\left(
     \frac{2^{\beta-\alpha}C_{\mathrm{hist}}}{1-2^{-\alpha}}
     +\frac{C_{\mathrm{mg}}}{1-2^{-\beta/2}}
   \right).
\]
This is a conclusion about the parent-action running averages;
throughout the argument the child estimates are assumed to
concentrate only at deterministic counts.
\end{proof}

\begin{lemma}[Leaf concentration]\label{lem:leaf_concentration}
    For any leaf node $s$ at depth $H$ in $\cG$, the value estimate $\Vhat_n(s)$ (average of $n$ i.i.d.\ calls to $\pi_0$) satisfies:
    \begin{equation}
        \Vhat_n(s) \conc{\alpha_H}{\beta_H} \Vtil(s),
    \end{equation}
    where $\alpha_H, \beta_H$ can be chosen to satisfy $\alpha_H \leq \beta_H/2$ with $\beta_H > 2$, and $\Vtil(s) = V_0(s)$ is the playout value.
\end{lemma}

\begin{proof}
    Since $\Vhat_n(s) = \frac{1}{n}\sum_{i=1}^n X_i$ where $X_i \sim \pi_0(s)$ are i.i.d.\ with mean $V_0(s)$ and bounded in $[0, R_{\max}/(1-\gamma)]$, Hoeffding's inequality gives $\PP(|\Vhat_n(s) - V_0(s)| > \varepsilon) \leq 2\exp(-2n\varepsilon^2/L^2)$, which implies $(\alpha,\beta)$-concentration for any $\alpha \leq \beta/2$ with appropriate constants.
\end{proof}

\subsection{Proof of Theorem~\ref{thm:rate}}

\begin{lemma}[Weighted action errors at adaptive counts]
\label{lem:weighted_action_errors}
Let $K<\infty$ and $\beta>0$. Suppose that, for every action $a$, there is a
deterministic target $\mu_a$ and a common finite constant $A$ such that
\begin{equation}\label{eq:node_arm_maximal_premise}
 \PP\left(\sup_{k\ge r}|\Qhat_{a,k}-\mu_a|>\varepsilon\right)
 \le A r^{-\beta/2}\varepsilon^{-\beta},
 \qquad r\ge1,\quad\varepsilon>0.
\end{equation}
Here and below a suffix starts at an integer $r$.
For any adaptive action counts $N_a(n)$ with $\sum_aN_a(n)=n$,
there is a finite constant $A_W$ such that
\begin{equation}\label{eq:weighted_action_error_bound}
 \PP\left(\sup_{n\ge r}\sum_{a=1}^K\frac{N_a(n)}{n}
       |\Qhat_{a,N_a(n)}-\mu_a|>\varepsilon\right)
 \le A_W r^{-\beta/2}\varepsilon^{-\beta}.
\end{equation}
A term with $N_a(n)=0$ is defined to be zero. No independence
between estimates and action counts is required.
\end{lemma}

\begin{proof}
For each action, put
\[
 Z_{a,r}:=\sup_{k\ge1}\frac{k}{\max\{r,k\}}
                         |\Qhat_{a,k}-\mu_a|.
\]
The part with $k\ge r$ is bounded directly by
\eqref{eq:node_arm_maximal_premise}. For $k<r$, use the blocks
$2^j\le k<2^{j+1}$. If the weighted error exceeds $\varepsilon$
inside such a block, then
\[
 \sup_{k\ge2^j}|\Qhat_{a,k}-\mu_a|
       >\varepsilon r/2^{j+1}.
\]
The probability is at most
\[
 A(2^j)^{-\beta/2}(\varepsilon r/2^{j+1})^{-\beta}
 = A2^\beta r^{-\beta}\varepsilon^{-\beta}2^{j\beta/2}.
\]
Summing over $0\le j\le\lfloor\log_2r\rfloor$ shows that
\[
 \PP(Z_{a,r}>\varepsilon)
 \le A_Z r^{-\beta/2}\varepsilon^{-\beta}
\]
for a finite constant $A_Z$. For all $n\ge r$, every nonzero count
$k=N_a(n)$ satisfies $k/n\le k/\max\{r,k\}$. Consequently the
sum in \eqref{eq:weighted_action_error_bound} is at most
$\sum_a Z_{a,r}$. A union bound with threshold $\varepsilon/K$
proves the claim.
\end{proof}

\begin{lemma}[Power-mean node concentration]
\label{lem:power_mean_node_concentration}
Suppose that $0\le\Qhat_{a,k},\mu_a\le L$ for a finite $L>0$,
that $K<\infty$, and that the action sequences satisfy
\eqref{eq:node_arm_maximal_premise} for some $\beta>8$.
Let $\mu_\star=\max_a\mu_a$, fix $1\le p<\infty$ and $C>0$,
and fix an offset $\sigma\in\{0,1\}$.
After $t$ completed action decisions, let $\sum_aN_a(t)=t$.
After selecting each unvisited action once, choose the next action by
\begin{equation}\label{eq:node_local_index}
 a_{t+1}\in\argmax_a
 \left\{\Qhat_{a,N_a(t)}
       +C\frac{(t+\sigma)^{1/4}}{N_a(t)^{1/2}}\right\}.
\end{equation}
Ties may be resolved arbitrarily. Define the stored value after
$t\ge1$ action decisions by
\begin{equation}\label{eq:node_stored_value}
 \Vhat_t^{(\sigma)}
 =\left(\sum_{a=1}^K\frac{N_a(t)}{t+\sigma}
                  \Qhat_{a,N_a(t)}^p\right)^{1/p}.
\end{equation}
Put $b=\beta/4$ and $\zeta=b-1>1$. There is a finite constant
$A_V$ such that
\begin{equation}\label{eq:node_maximal_conclusion}
 \PP\left(\sup_{t\ge r}|\Vhat_t^{(\sigma)}-\mu_\star|
                       >\varepsilon\right)
 \le A_V r^{-\zeta/2}\varepsilon^{-\zeta},
 \qquad r\ge1,\quad\varepsilon>0.
\end{equation}
In particular, the stored value has fixed-count concentration
with exponents $\alpha'=\zeta/2$ and $\beta'=\zeta$.
The constant may depend on $K,C,L,p,\beta,A$, the positive gaps
$\mu_\star-\mu_a$, and the positive target $\mu_\star$ when it
is nonzero, but not on $r$ or $\varepsilon$. A unique optimal
action is not required.
\end{lemma}

\begin{proof}
We prove the maximal statement directly.

\paragraph{A single event controls all late action selections.}
For an integer $u\ge1$, define
\[
 E_u:=\bigcap_{a=1}^K\bigcap_{k\ge1}
 \left\{|\Qhat_{a,k}-\mu_a|
       \le\frac C2\frac{\max\{u,k\}^{1/4}}{k^{1/2}}\right\}.
\]
The fixed-count consequences of
\eqref{eq:node_arm_maximal_premise} imply
\begin{align}
 \PP(E_u^c)
 &\le KA(2/C)^\beta\sum_{k\ge1}\max\{u,k\}^{-b}\nonumber\\
 &\le KA(2/C)^\beta
          \left(1+\frac1{b-1}\right)u^{1-b}.
 \label{eq:node_late_event_probability}
\end{align}
At a decision time $t\ge\max\{u,K\}$, every action has a
positive count $k\le t$. On $E_u$, its estimation error is
at most half its bonus in \eqref{eq:node_local_index}.
Thus a selected suboptimal action $a$, with
$\Delta_a=\mu_\star-\mu_a>0$, satisfies
\[
 \Delta_a\le\frac{3C}{2}
                  \frac{(t+\sigma)^{1/4}}{N_a(t)^{1/2}}.
\]
Looking at its last selection among the first $n$ decisions gives
\[
 N_a(n)\le u+K+1+
                 \left(\frac{3C}{2\Delta_a}\right)^2\sqrt{n+1}.
\]
For a suffix starting at $r$, choose $u=\lceil\sqrt r\rceil$.
There is consequently a finite constant $D$ such that, on $E_u$,
simultaneously for every $n\ge r$,
\begin{equation}\label{eq:node_true_gap_weight}
 G_n:=\sum_a\frac{N_a(n)}n(\mu_\star-\mu_a)
       \le D r^{-1/2}.
\end{equation}
One possible choice is
\[
 D=(K+3)\sum_{a:\Delta_a>0}\Delta_a
   +\frac{9\sqrt2 C^2}{4}\sum_{a:\Delta_a>0}\Delta_a^{-1},
\]
with empty sums equal to zero. Equation
\eqref{eq:node_late_event_probability} gives
$\PP(E_u^c)\le B_0r^{-\zeta/2}$.

\paragraph{Power means and the normalization offset.}
Introduce the normalized value and weighted estimation error
\[
 F_n:=\left(\sum_a\frac{N_a(n)}n
                           \Qhat_{a,N_a(n)}^p\right)^{1/p},
 \qquad
 W_n:=\sum_a\frac{N_a(n)}n
                           |\Qhat_{a,N_a(n)}-\mu_a|.
\]
The two normalizations obey
\begin{equation}\label{eq:node_offset_error}
 |\Vhat_n^{(\sigma)}-F_n|
 \le L\left[1-\left(\frac n{n+\sigma}\right)^{1/p}\right]
 \le\frac{L\sigma}{n+\sigma}.
\end{equation}
This uses $p\ge1$ and $0\le F_n\le L$.

\paragraph{The case $\mu_\star>0$.}
For $x\ge0$, $y>0$, and $p\ge1$,
$|x^p-y^p|\ge y^{p-1}|x-y|$. On $[0,L]$ the power map also
satisfies $|x^p-y^p|\le pL^{p-1}|x-y|$. Inserting
$\sum_a(N_a(n)/n)\mu_a^p$ between $F_n^p$ and
$\mu_\star^p$ therefore yields
\begin{equation}\label{eq:node_power_transfer}
 |F_n-\mu_\star|
 \le B_p(W_n+G_n),\qquad
 B_p=p(L/\mu_\star)^{p-1}.
\end{equation}
This estimate permits any number of tied optimal actions.
Combine \eqref{eq:node_true_gap_weight}--\eqref{eq:node_power_transfer}
with Lemma~\ref{lem:weighted_action_errors}. For finite constants
$D_1,B_1,B_2$, whenever
$\varepsilon\sqrt r\ge\max\{1,2D_1\}$,
\begin{equation}\label{eq:node_tail_intermediate}
 \PP\left(\sup_{n\ge r}|\Vhat_n^{(\sigma)}-\mu_\star|
                  >\varepsilon\right)
 \le B_1r^{-\zeta/2}
           +B_2r^{-\beta/2}\varepsilon^{-\beta}.
\end{equation}
For example, the deterministic threshold can use
$D_1=B_pD+L$.
The error is bounded by $L$, so it suffices to consider
$0<\varepsilon\le L$. The first term on the right of
\eqref{eq:node_tail_intermediate} is at most
$B_1L^\zeta r^{-\zeta/2}\varepsilon^{-\zeta}$.
Since $\beta>\zeta$ and $\varepsilon\sqrt r$ is bounded below,
the second is at most a constant times
$r^{-\zeta/2}\varepsilon^{-\zeta}$ as well. If
$\varepsilon\sqrt r$ is below the displayed threshold, enlarge
the constant so that this last bound is at least one. This proves
\eqref{eq:node_maximal_conclusion} when $\mu_\star>0$.

\paragraph{The case $\mu_\star=0$.}
Nonnegativity implies $\mu_a=0$ for every action. On $E_u$,
when action $b$ is selected at time $t\ge\max\{u,K\}$,
comparison with any action $a$ gives
\[
 \frac C2\frac{(t+\sigma)^{1/4}}{N_a(t)^{1/2}}
 \le I_a(t)\le I_b(t)
 \le\frac{3C}{2}\frac{(t+\sigma)^{1/4}}{N_b(t)^{1/2}},
\]
where $I_a(t)$ is its index. Hence $N_b(t)\le9N_a(t)$.
Take $r\ge r_0:=16(K+1)^2$ and any $n\ge r$. At least one
action $b$ has $N_b(n)\ge n/K$. Just before its last selection,
its count is $N_b(n)-1\ge n/K-1$. That selection occurs after
$\max\{u,K\}$, because $u=\lceil\sqrt r\rceil$. The comparison
above therefore implies, for every action,
\[
 N_a(n)\ge\frac{n/K-1}{9}\ge\frac{n}{18K}.
\]
Since $\Vhat_n^{(\sigma)}\le F_n\le\max_a\Qhat_{a,N_a(n)}$,
on $E_u$,
\[
 \sup_{n\ge r}\Vhat_n^{(\sigma)}
 \le\max_a\sup_{k\ge\lfloor r/(18K)\rfloor}|\Qhat_{a,k}|.
\]
The input maximal bounds, together with $\PP(E_u^c)$, give
\eqref{eq:node_tail_intermediate}. The same threshold and
boundedness argument proves \eqref{eq:node_maximal_conclusion}
for $r\ge r_0$. The finite range $1\le r<r_0$ is absorbed into
the constant, again using the error bound $L$.
\end{proof}

\begin{remark}[First initialization and the global node count]
\label{rem:node_initialization_count}
At the root, $\sigma=0$. A newly created nonroot internal node
in Algorithm~\ref{alg:gs_power_uct} first returns one bounded
evaluator value and records its first global visit; thereafter
$\sigma=1$ and its global count is $t+1$ after $t$ action
decisions. Lemma~\ref{lem:power_mean_node_concentration}
explicitly covers the resulting unnormalized weight sum and the
shifted bonus. The first returned value has error at most $L$.
For a global-count suffix starting at $r\ge2$, the subsequent
decision counts satisfy $t\ge r-1\ge r/2$, so the lemma transfers
to global visit counts with only a constant factor. The case
$r=1$ is covered by the bounded first return and the suffix bound
starting at the first action decision. This accounts for the
initialization convention without treating its mean as the
internal truncated Bellman target.
\end{remark}

\paragraph{Continued executions and local counts.}
For the analysis, continue the search indefinitely with the same fixed
horizon and parameters. Every reachable nonterminal depth-augmented
node is then visited infinitely often almost surely, and every action
at such a node is selected infinitely often. Indeed, at a node with
infinitely many visits, suppose some action counts remained bounded.
After dividing all indices by the fourth root of the node count,
the bounded $Q$ estimates vanish. The bonus coefficients of actions
with diverging counts tend to zero, whereas those of the finitely
visited actions stay strictly positive. Hence every sufficiently
late selection would have to be among the finitely visited actions,
a contradiction. Starting at the root, the strong law for each
fresh local successor stream then shows inductively in depth that
every positive-probability reachable child is visited infinitely
often. Only finitely many such queries occur at fixed horizon.
The first initialization visit does not affect this argument.
Consequently, all count-indexed snapshots and suffix suprema used
below are defined on a common probability-one event; this continuation
does not change the law of any finite-budget run.

We now establish the depth-wise induction.

\begin{theorem}[Convergence of GS-Power-UCT]\label{thm:main}
Apply Algorithm~\ref{alg:gs_power_uct} under
Assumption~\ref{assump:standing}, with fixed $1\le p<\infty$,
$C>0$, finite $H$, and the tuning in Table~\ref{tab:conditions}.
Its depth-augmented search graph is a DAG. In the continued
fixed-horizon execution, the following bounds hold for each fixed
reachable node. Here $k$ is a deterministic local visit or action
count, not the global simulation budget.
\begin{enumerate}
 \item[(i)] At every fixed reachable node $(s,h)$, $0\le h\le H$,
 the estimate after $k$ global visits to this node satisfies
 \[
  \Vhat_k(s,h)\xrightarrow[k\to\infty]{\alpha_h,\beta_h}\Vtil(s,h).
 \]
 \item[(ii)] At every fixed reachable nonterminal node $(s,h)$,
 $h<H$, and action $a\in\cA_s$, the estimate after $k$ updates
 of this node-action pair satisfies
 \[
  \Qhat_k(s,h,a)
      \xrightarrow[k\to\infty]{\alpha_{h+1},\beta_{h+1}}
      \Qtil(s,h,a).
 \]
\end{enumerate}
\end{theorem}

\begin{proof}
If $L=0$, all estimates and targets are zero and the claim is immediate.
Otherwise, we prove by strong induction on depth $h$, from $H$ down to $0$. The induction is well-defined because the depth-augmented graph is a DAG by Proposition~\ref{prop:dag}: every edge connects depth $h$ to depth $h+1$, so all children of a depth-$h$ node have depth exactly $h+1$.

\textbf{Base case ($h = H$):} By Lemma~\ref{lem:leaf_concentration}, for any node $(s, H)$ at depth $H$:
\begin{equation}
    \Vhat_n(s, H) \conc{\alpha_H}{\beta_H} \Vtil(s,H) = V_0(s).
\end{equation}

Physical-terminal nodes at any earlier depth also have identically zero estimates and targets, so their concentration is immediate.

\textbf{Inductive step:} Suppose (i) and (ii) hold for all depths $h+1, h+2, \ldots, H$. Consider a nonterminal node $(s, h)$ at depth $h$.

\emph{Step 1: Q-value concentration.} For any action $a \in \cA_s$, the successor nodes are of the form $(s', h+1)$ with $s' \sim P(\cdot|s,a)$, which have depth exactly $h+1$ by the depth-augmented construction. By the inductive hypothesis:
\begin{equation}
    \Vhat_{n}(s', h+1) \conc{\alpha_{h+1}}{\beta_{h+1}} \Vtil(s',h+1), \quad \text{for all successors } s'.
\end{equation}
Crucially, this concentration targets the correct value $\Vtil(s',h+1)$ because every visit to node $(s', h+1)$ occurs at depth $h+1$ --- the depth-augmentation guarantees no mixing of estimates from different effective horizons.

The parameter conditions give
$\alpha_{h+1}>b_{h+1}>2$.
Successive backups through a fixed parent--action pair use
distinct global visit counts for each child. Thus the
fixed-count inductive hypothesis satisfies the premises
of Lemma~\ref{lem:graph_Q_concentration}.

By Lemma~\ref{lem:graph_Q_concentration}:
\begin{equation}
    \Qhat_n(s,h,a) \conc{\alpha_{h+1}}{\beta_{h+1}} \Qtil(s,h,a), \quad \forall a \in \cA_s.
\end{equation}
This establishes (ii) at depth $h$.

\emph{Step 2: Value concentration via power mean.}
The stronger conclusion~\eqref{eq:graph_Q_maximal} of Lemma~\ref{lem:graph_Q_concentration}
also gives maximal action-estimate concentration with
$\alpha_{h+1}=\beta_{h+1}/2$. Apply
Lemma~\ref{lem:power_mean_node_concentration} with
$\beta=\beta_{h+1}$, and account for the first initialization
using Remark~\ref{rem:node_initialization_count}. This gives
\[
    \Vhat_n(s,h)\conc{\alpha_h}{\beta_h}\Vtil(s,h),
    \qquad
    \beta_h=\beta_{h+1}/4-1,\quad\alpha_h=\beta_h/2,
\]
exactly as specified in Table~\ref{tab:conditions}.
The argument permits tied optimal actions and arbitrary $C>0$.

This establishes (i) at depth $h$, completing the induction.
\end{proof}

\begin{proof}[Proof of Theorem~\ref{thm:rate}]
By Theorem~\ref{thm:main}, for a finite constant $c_0$,
\[
 \PP\bigl(|\Vhat_n(s_0,0)-\Vtil(s_0,0)|>x\bigr)
 \le c_0 n^{-\alpha_0}x^{-\beta_0}.
\]
Split the tail integral at $x_0=n^{-\alpha_0/\beta_0}$.
Since $\beta_0>1$,
\begin{align*}
 \EE|\Vhat_n(s_0,0)-\Vtil(s_0,0)|
 &\le x_0+c_0n^{-\alpha_0}\int_{x_0}^\infty x^{-\beta_0}\,dx\\
 &\le\left(1+\frac{c_0}{\beta_0-1}\right)n^{-\alpha_0/\beta_0}
 =O(n^{-1/2}),
\end{align*}
because Table~\ref{tab:conditions} gives $\alpha_0/\beta_0=1/2$.
The same bound implies the absolute bias bound by Jensen's inequality.
\end{proof}

\subsection{Proof of Theorem~\ref{thm:graph_nonworsening}}
\label{app:proof_graph_nonworsening}

\begin{proof}
Fix the trajectory records. A recorded history prefix ending at
depth $h$ identifies one tree copy $U_{i,h}\in\cT_h$, and its
projection is the graph node $V_{i,h}=\phi_h(U_{i,h})$.
If trajectory $i$ does not reach depth $h$, all depth-$h$ visit
indicators below are zero. Since each tree copy has a unique
projection, for every $v\in\cN_h$,
\[
    \mathbf1\{V_{i,h}=v\}
       =\sum_{u\in\mathcal C(v)}\mathbf1\{U_{i,h}=u\}.
\]
Summing over the $n$ recorded trajectories gives
\begin{equation}
\label{eq:appendix_graph_tree_visit_identity}
    T_v^{\cG}(n)
       =\sum_{u\in\mathcal C(v)}T_u^{\cT}(n).
\end{equation}
Each summand is nonnegative, so
\begin{equation}
\label{eq:appendix_graph_tree_visit_ineq}
    T_v^{\cG}(n)\ge T_u^{\cT}(n),
    \qquad u\in\mathcal C(v).
\end{equation}
Restricting the indicators in both representations to the same
completed internal backups proves the corresponding count identity
under that counting convention.

If action $a$ is selected on the recorded visit, refining the same
indicators by that action gives
\[
    \mathbf1\{V_{i,h}=v,\ a_{i,h}=a\}
       =\sum_{u\in\mathcal C(v)}
          \mathbf1\{U_{i,h}=u,\ a_{i,h}=a\}.
\]
These indicators are zero when no action is selected at that visit.
Consequently,
\begin{equation}
\label{eq:appendix_graph_tree_action_identity}
    T_{v,a}^{\cG}(n)
       =\sum_{u\in\mathcal C(v)}T_{u,a}^{\cT}(n),
\end{equation}
and
\begin{equation}
\label{eq:appendix_graph_tree_action_ineq}
    T_{v,a}^{\cG}(n)\ge T_{u,a}^{\cT}(n),
    \qquad u\in\mathcal C(v),\ a\in\cA_s,
    \quad v=(s,h).
\end{equation}

The fibers $\{\mathcal C(v):v\in\cN_h\}$ partition $\cT_h$.
Thus the common total of counted depth-$h$ visits is
\begin{equation}
\label{eq:appendix_common_depth_total}
    N_h(n)
       =\sum_{v\in\cN_h}T_v^{\cG}(n)
       =\sum_{u\in\cT_h}T_u^{\cT}(n).
\end{equation}
If $N_h(n)>0$, both representation sets are nonempty and the
denominator in the average-count ratio is positive. Dividing
$N_h(n)/|\cN_h|$ by $N_h(n)/|\cT_h|$ gives
\begin{equation}
\label{eq:appendix_average_sharing}
    \frac{
       |\cN_h|^{-1}\sum_{v\in\cN_h}T_v^{\cG}(n)
    }{
       |\cT_h|^{-1}\sum_{u\in\cT_h}T_u^{\cT}(n)
    }
    =\frac{|\cT_h|}{|\cN_h|}
    =\frac1{\tau_h}.
\end{equation}
When $N_h(n)=0$, all counts in this identity vanish and no ratio
is asserted. Finally, projection from the rooted unrolled tree
onto the graph is surjective on nodes and on represented
transition edges, proving the cardinality comparison.

All statements follow from identities between fixed records.
No concentration inequality is evaluated at a random count, and
no assertion that these records follow the action-selection rule
of an independently run tree planner is required.
\end{proof}

\subsection{A Fixed-Horizon Bias Example for Full State Merging}
\label{app:cycles}

Depth augmentation assigns each stored value a single remaining horizon.
The following example shows that, even with the corrected count and
backup rules, full merging can converge to a different value from the
ordinary finite-horizon target.

\begin{proposition}[Full merging can retain fixed-horizon bias]
\label{prop:bias}
Consider one nonterminal state $s$ with one action, a deterministic
self-loop, reward $1$, discount $\gamma=1/2$, and deterministic evaluator
$V_0(s)=0$. Run Algorithm~\ref{alg:full_merge} with horizon $H=2$.
For any finite $p\ge1$, its root estimate satisfies
\begin{equation}
\label{eq:full_merge_cyclic_bias_limit}
    \lim_{n\to\infty}\Vhat_n^2(s)=\frac43,
    \qquad
    \Vtil(s,0)=\frac32,
    \qquad
    \Vtil(s,1)=1.
\end{equation}
Consequently, the limiting root error relative to the truncated target
is $1/6$. Both internal depths receive exactly one completed update
per simulation.
\end{proposition}

\begin{proof}
With one action, the power mean equals the action mean. Let $q_n$
be this shared mean after $n$ complete simulations, with $q_0=0$.
There are exactly $2n$ completed internal updates at that time;
the cutoff evaluation has separate statistics and contributes no
internal count.

During simulation $n$, the depth-one call first backs up the target
$1+(1/2)V_0(s)=1$. Write $u_n$ for the shared mean immediately
after that update. The depth-zero call then uses the scalar return
$u_n$ and backs up $1+(1/2)u_n$. Hence the exact recursion is
\begin{equation}
\label{eq:full_merge_cyclic_bias_recursion}
\begin{aligned}
    u_n&=\frac{(2n-2)q_{n-1}+1}{2n-1},\\
    q_n&=\frac{(2n-1)u_n+1+\tfrac12u_n}{2n}.
\end{aligned}
\end{equation}
Equivalently,
\[
    q_n=(1-a_n)q_{n-1}+a_n b_n,
    \qquad
    a_n=\frac{3n-1}{2n(2n-1)},
    \qquad
    b_n=\frac{8n-3}{2(3n-1)}.
\]
Here $0<a_n\le1$, $\sum_{n\ge1}a_n=\infty$, and $b_n\to4/3$.
To verify convergence directly, fix $\epsilon>0$ and choose $N$
so that $|b_n-4/3|\le\epsilon$ for all $n>N$. Iterating the last
recursion gives
\[
    |q_n-4/3|
    \le
    \left(\prod_{j=N+1}^{n}(1-a_j)\right)|q_N-4/3|
    +\epsilon\left(1-\prod_{j=N+1}^{n}(1-a_j)\right).
\]
The product tends to zero because $\sum_j a_j=\infty$.
Since $\epsilon$ is arbitrary, $q_n\to4/3$.
The two truncated targets are $1+\gamma=3/2$ and $1$, respectively,
which proves the stated root error.
\end{proof}

\begin{remark}[Depth gaps and recursive feedback]
The quantity $\delta(s)$ measures differences between truncated
Bellman targets. The bound on this quantity below concerns those
target differences; it does not establish an error bound for a
recursively updated full-merge estimator. In the example,
the two internal depths have equal update frequencies, but the
limit $4/3$ is different from their target average
$(3/2+1)/2=5/4$. Recursive child returns therefore cannot generally
be treated as samples with the corresponding truncated target as
their conditional mean. Theorem~\ref{thm:full_merge_rate} instead
bounds the corrected algorithm directly through its completed
backups and the evaluator-dependent quantity $B_H$.
\end{remark}

\begin{remark}[Depth-indexed storage]
A single physical-state record may store separate decision counts,
action estimates, and value estimates for each depth. If selection
and backup use the corresponding depth-specific statistics, this
implements the same statistical separation as the key $(s,h)$.
Sharing the physical record in this way does not merge its
depth-specific value targets.
\end{remark}

\begin{lemma}[Cross-depth gap bound]
\label{lem:cross_depth_gap_bound}
Assume $\gamma\in(0,1)$ and let
\[
    \Vtil(s,h)=(\mathcal B^{H-h}V_0)(s)
\]
be the depth-$h$ truncated Bellman value. For any state $s$ and any two depths
$h_1<h_2$, we have
\begin{equation}
\label{eq:two_depth_gap_bound}
    \left|
        \Vtil(s,h_1)-\Vtil(s,h_2)
    \right|
    \le
    \frac{
        \gamma^{H-h_2}
        -
        \gamma^{H-h_1}
    }{1-\gamma}
    \,
    \|\mathcal B V_0 - V_0\|_\infty .
\end{equation}
Consequently, for
\[
    \delta(s)
    =
    \max_{h_1,h_2\in\mathcal D(s)}
    |\Vtil(s,h_1)-\Vtil(s,h_2)|,
\]
we have
\begin{equation}
\label{eq:delta_s_bound_residual}
    \delta(s)
    \le
    \frac{
        \gamma^{H-h_{\max}(s)}
        -
        \gamma^{H-h_{\min}(s)}
    }{1-\gamma}
    \,
    \|\mathcal B V_0 - V_0\|_\infty .
\end{equation}
Moreover,
\begin{equation}
\label{eq:delta_s_bound_vstar}
    \delta(s)
    \le
    \frac{
        (1+\gamma)
        \left(
            \gamma^{H-h_{\max}(s)}
            -
            \gamma^{H-h_{\min}(s)}
        \right)
    }{1-\gamma}
    \,
    \|\Vstar - V_0\|_\infty .
\end{equation}
\end{lemma}

\begin{proof}
The case $h_1=h_2$ is trivial, so assume $h_1<h_2$. Let
\[
    k_1=H-h_1,
    \qquad
    k_2=H-h_2,
    \qquad
    d=k_1-k_2=h_2-h_1.
\]
Then $k_1>k_2$ and
\[
    \Vtil(\cdot,h_1)
    =
    \mathcal B^{k_1}V_0
    =
    \mathcal B^{k_2}(\mathcal B^d V_0),
    \qquad
    \Vtil(\cdot,h_2)
    =
    \mathcal B^{k_2}V_0 .
\]
Since $\mathcal B$ is a $\gamma$-contraction in sup norm,
\begin{align}
    \|\Vtil(\cdot,h_1)-\Vtil(\cdot,h_2)\|_\infty
    &=
    \left\|
        \mathcal B^{k_2}(\mathcal B^d V_0)
        -
        \mathcal B^{k_2}V_0
    \right\|_\infty \\
    &\le
    \gamma^{k_2}
    \left\|
        \mathcal B^d V_0 - V_0
    \right\|_\infty .
\end{align}
Now use the telescoping decomposition
\[
    \mathcal B^d V_0 - V_0
    =
    \sum_{j=0}^{d-1}
    \left(
        \mathcal B^{j+1}V_0
        -
        \mathcal B^j V_0
    \right).
\]
Again by contraction,
\[
    \left\|
        \mathcal B^{j+1}V_0
        -
        \mathcal B^j V_0
    \right\|_\infty
    \le
    \gamma^j
    \|\mathcal B V_0 - V_0\|_\infty .
\]
Therefore,
\begin{align}
    \left\|
        \mathcal B^d V_0 - V_0
    \right\|_\infty
    &\le
    \sum_{j=0}^{d-1}
    \gamma^j
    \|\mathcal B V_0 - V_0\|_\infty \\
    &=
    \frac{1-\gamma^d}{1-\gamma}
    \|\mathcal B V_0 - V_0\|_\infty .
\end{align}
Combining the last two displays gives
\[
    \|\Vtil(\cdot,h_1)-\Vtil(\cdot,h_2)\|_\infty
    \le
    \frac{
        \gamma^{k_2}
        -
        \gamma^{k_2+d}
    }{1-\gamma}
    \|\mathcal B V_0 - V_0\|_\infty .
\]
Since $k_2=H-h_2$ and $k_2+d=H-h_1$, this is exactly
\[
    \|\Vtil(\cdot,h_1)-\Vtil(\cdot,h_2)\|_\infty
    \le
    \frac{
        \gamma^{H-h_2}
        -
        \gamma^{H-h_1}
    }{1-\gamma}
    \|\mathcal B V_0 - V_0\|_\infty .
\]
Taking the maximum over $h_1,h_2\in\mathcal D(s)$ gives
\eqref{eq:delta_s_bound_residual}, because the right-hand side is maximized by
$h_1=h_{\min}(s)$ and $h_2=h_{\max}(s)$.

Finally, since $\Vstar=\mathcal B\Vstar$,
\begin{align}
    \|\mathcal B V_0 - V_0\|_\infty
    &=
    \|(\mathcal B V_0-\mathcal B\Vstar)+(\Vstar-V_0)\|_\infty \\
    &\le
    \|\mathcal B V_0-\mathcal B\Vstar\|_\infty
    +
    \|\Vstar-V_0\|_\infty \\
    &\le
    (1+\gamma)\|\Vstar-V_0\|_\infty .
\end{align}
Substituting this into \eqref{eq:delta_s_bound_residual} proves
\eqref{eq:delta_s_bound_vstar}.
\end{proof}

\begin{definition}[Boundary-mixture reference operator]
\label{def:empirical_full_merge_target}
After $n$ complete simulations, let $T_{s,a,h}(n)$ count completed
internal backups selected at state $s$, action $a$, and depth $h$.
Set $T_{s,a}(n)=\sum_hT_{s,a,h}(n)$ and
$\omega_{n,h}(s,a)=T_{s,a,h}(n)/T_{s,a}(n)$ when the denominator
is positive; otherwise choose arbitrary probability weights over
$h=0,\ldots,H-1$. For a bounded function $V:\cS\to\RR$, define
\[
 B_hV(x)=
 \begin{cases}
 V_0(x),&h+1=H\text{ or }x\text{ is physically terminal},\\
 V(x),&\text{otherwise}.
 \end{cases}
\]
At nonterminal states, put
\begin{equation}\label{eq:empirical_full_merge_operator}
 (\mathcal T_{F,n}^H V)(s)
 =\max_{a\in\cA_s}\sum_{h=0}^{H-1}\omega_{n,h}(s,a)
   \sum_{x}P(x\mid s,a)[R(s,a,x)+\gamma B_hV(x)],
\end{equation}
and at physical terminals set $(\mathcal T_{F,n}^H V)(s)=V_0(s)=0$.
This is a $\gamma$-contraction on the bounded functions with the
supremum norm. It therefore has a unique bounded fixed point
$\bar V_n$. Its nonterminal action targets are
\begin{equation}\label{eq:empirical_full_merge_Q}
 \bar Q_n(s,a)=\sum_{h=0}^{H-1}\omega_{n,h}(s,a)
   \sum_xP(x\mid s,a)[R(s,a,x)+\gamma B_h\bar V_n(x)].
\end{equation}
This operator records the empirical depth frequencies and boundary
convention. It is not the algorithm's historical backup operator:
intermediate child estimates and first-creation evaluations also enter
those backups. No convergence to $\bar V_n$ is assumed or used in
Theorem~\ref{thm:full_merge_rate}.
\end{definition}

\subsection{Proof of Theorem~\ref{thm:full_merge_rate}}

The proof uses an all-time noise bound, a finite accuracy bootstrap,
and a final induction over completed backups. Local count statements
below apply whenever the specified local count is attained; no
assumption that every state is visited infinitely often is needed.

\begin{lemma}[Polynomial lower bound on local action counts]
\label{lem:full_merge_forced_counts}
For bounded action estimates in $[0,L]$ and the exploration rule in
Theorem~\ref{thm:full_merge_rate}, there are deterministic constants
$c>0$ and $t_{\min}<\infty$ such that, at every node and every
attained completed local count $t\ge t_{\min}$,
\[
    N_a(t)\ge c t^{\rho_-}
    \quad\text{for all actions }a,
    \qquad
    \rho_-=\min_{h<H}\frac{u_h}{v_h}>0.
\]
The constants can be chosen uniformly over the finite reachable
graph. Sequential depth-$H$ recursion supplies a fixed bound on
selection-to-completion delay, measured in completed updates.
\end{lemma}

\begin{proof}
Let $D$ bound that delay, including the selected update.
Each recursive call follows one child path, and depth increases
at each call. Hence at most $H$ internal backups complete between
a selection and its own completion; we may take $D=H$.
Choose $c>0$ sufficiently small that, for every depth $h$,
\[
    \frac{C_h\,2^{-u_h}}{c^{v_h}}>L+2C_{\max},
    \qquad C_{\max}=\max_{h<H}C_h.
\]
Suppose that at completed local time $t$ an action has
$N_a(t)<ct^{\rho_-}$.
Consider selections made before this observation whose
pre-selection local clock lies in $[t/2,t]$.
Throughout this window, this action is either unvisited, or its index is greater than
$L+2C_{\max}$: use
$u_h-\rho_-v_h\ge0$ and monotonicity of its count.

A regular selected action at such a time must therefore have
exploration bonus greater than $2C_{\max}$. Its pre-selection count
is at most
\[
    b\,t^\rho
\]
for a deterministic constant $b$, uniformly over the finite set of
depths. A forced zero-count selection also satisfies this bound.
Its count at completion is at most $bt^\rho+D$.
All completions with local indices above $t/2+D+1$ arise from
selections in this window.
There are at least $t/2-D-1$ such completions, but at most
$K(bt^\rho+D+1)$ completions can have the indicated per-action
count bound.
Since $\rho<1$, this is impossible for sufficiently large $t$.
This proves the result, with integer roundings absorbed into
$t_{\min}$.
\end{proof}

\begin{lemma}[Uniform noise control and child-count propagation]
\label{lem:full_merge_pathwise_propagation}
Fix $\vartheta\in(0,1/2)$. There is a nonnegative random variable $Z\ge1$
with finite moments of every order, and deterministic constants
absorbed into $Z$, with the following properties simultaneously
for all observed state--action counts and all completed updates.

For $k\ge1$ and $E\ge0$, if every internal child estimate whose local completed count is
at least $k$ has error at most $E$ relative to $V^\star$, then
\begin{equation}
\label{eq:full_merge_coarse_Q}
    |\widehat Q_{s,a,m}-Q^\star(s,a)|
    \le
    \gamma\max(E,\varepsilon_H)
    +\frac{\gamma LMk}{m}
    +Z m^{-\vartheta}.
\end{equation}

For $B,A\ge0$, if all preceding internal child updates satisfy
\[
    |\widehat V_{x,j}-V^\star(x)|\le B+A j^{-\eta},
    \qquad 0<\eta<1,
\]
then the current action mean satisfies
\begin{equation}
\label{eq:full_merge_fine_Q}
    |\widehat Q_{s,a,m}-Q^\star(s,a)|
    \le
    \gamma\max(B,\varepsilon_H)
    +\frac{\gamma A M^\eta}{1-\eta}m^{-\eta}
    +Z m^{-\vartheta}.
\end{equation}
The coefficient $A$ may be a pathwise random constant.
\end{lemma}

\begin{proof}
For a fixed selected pair, the ideal target
$R_i+\gamma V^\star(S_i)$ has conditional mean $Q^\star(s,a)$.
Its centered partial sums form a bounded martingale-difference
stream in selection order.
For any bounded such stream $(X_i)$, the random variable
\[
    Z_{\vartheta}^{(X)}
    =
    \sup_{\substack{m\ge1\\m\ {\rm observed}}}m^{\vartheta-1}
       \left|\sum_{i=1}^m X_i\right|
\]
has finite moments of every order, taking a zero supremum if no
query occurs. The supremum includes only observed counts when a
query stream terminates.
To make the stopping argument explicit, let $\tau_m$ be the time
of the $m$-th query in this stream, with $\tau_m=\infty$ if it
never occurs. Membership in the stream is determined before the
corresponding draw. For $|X_i|\le b$, stopping the exponential
supermartingale at $\tau_m$ gives
\[
    \PP\left(\tau_m<\infty,\,
        \left|\sum_{i=1}^m X_i\right|\ge x\right)
    \le 2\exp\left(-\frac{x^2}{2b^2m}\right).
\]
This is a joint-event bound, not a conditional bound given that
the $m$-th query occurs. Applying it with $x=z m^{1-\vartheta}$
and taking a union bound gives, for $z\ge1$,
\[
    \PP(Z_{\vartheta}^{(X)}>z)
    \le
    \sum_{m\ge1}2\exp(-c z^2m^{1-2\vartheta})
    \le C\exp(-c' z^2).
\]
Thus the selection rule need not be independent of past samples.

Completion-order prefixes and selection-order prefixes for one pair
differ by at most a bounded number of pending items.
Their normalized bounded sums differ by $O(HL/m)$, which is absorbed
into $Zm^{-\vartheta}$.

For fresh boundary returns, centered evaluator errors along the
subsequence assigned to each immediate parent pair form another
bounded martingale-difference stream: that parent is known before
the fresh evaluation is drawn. Its contribution to a parent average
with $m$ backups is at most a random constant times $m^{-\vartheta}$.

For a running-average evaluator, the same concentration argument
gives
\[
    |U_{x,j}-V_0(x)|\le Z_x j^{-\vartheta}
\]
simultaneously over its evaluator counts.
Each boundary call increments that count, so the indices used by
a fixed parent are distinct for each child.
Consequently, their total contribution is bounded using
\[
    \sum_{j=1}^{r}j^{-\vartheta}\le\frac{r^{1-\vartheta}}{1-\vartheta},
    \qquad
    \sum_{x=1}^{M}r_x^{1-\vartheta}\le M^{\vartheta} m^{1-\vartheta}.
\]
The single evaluation on first creating a node is included among
these boundary calls. The finite reachable set gives only finitely
many noise streams; a finite maximum, with deterministic scale
factors included, supplies the stated $Z$.

Every nonboundary child return follows a completed child update.
For any fixed parent pair and child, those returned global child
counts are distinct, even when the recursive path revisits a state.
Thus at most $Mk$ of the first $m$ parent backups can use an internal
child count smaller than $k$.
All such errors are bounded by $L$; later internal errors are
bounded by $E$, and boundary mean errors by $\varepsilon_H$.
This proves \eqref{eq:full_merge_coarse_Q}.

For the second claim, if $r_x$ internal returns use child $x$,
their varying errors are bounded by
\[
    A\sum_{j=1}^{r_x}j^{-\eta}
    \le \frac{A r_x^{1-\eta}}{1-\eta}.
\]
Since $\sum_x r_x\le m$,
\[
    \sum_x r_x^{1-\eta}\le M^\eta m^{1-\eta}.
\]
The constant part of every internal or boundary error is at most
$\max(B,\varepsilon_H)$.
Combining this with the already controlled centered noise proves
\eqref{eq:full_merge_fine_Q}.
No concentration bound is conditioned on a random child count.
\end{proof}

\begin{lemma}[Finite local accuracy implies a node-value bound]
\label{lem:full_merge_coarse_node}
Fix $h>0$, $m_0\ge1$, and $q\ge0$.
Suppose that every action estimate at every attained
count $m\ge m_0$ has error at most $q$ relative to its optimal
action value. There is a deterministic constant $C(h)<\infty$,
depending on the model bounds and exploration parameters but not
on $m_0$ or $q$, such that every node value at a completed local
count
\[
    t\ge C(h)(1+m_0^{1/\rho_-})
\]
has error at most $q+h$ relative to $V^\star$.
\end{lemma}

\begin{proof}
For large enough $t$, Lemma~\ref{lem:full_merge_forced_counts}
ensures that all present action counts exceed $m_0$.
Hence $V_t\le\max_a Q_a(t)\le V^\star+q$.

For the lower bound, choose a fixed $0<\theta<1/K$ with
$LK\theta\le h/4$.
Consider an action whose present count is at least $\theta t$,
and examine the selection of its last completed update.
At that selection, the completed local clock is at least
$\theta t-D$ and the selected action count is at least
$\theta t-D$.
Taking $t\ge C(h)(1+m_0^{1/\rho_-})$ makes every action count at
that selection at least $m_0$, by
Lemma~\ref{lem:full_merge_forced_counts}.
In particular, an optimal action then has estimate at least
$V^\star-q$.

The selected action's bonus is at most
\[
    \max_{h'<H}C_{h'}
       \frac{t^{u_{h'}}}{(\theta t-D)^{v_{h'}}},
\]
which tends to zero because each $u_{h'}<v_{h'}$.
Its arithmetic mean changes by at most
$DL/(\theta t-D)$ between selection and completion, with no later
update to that action. Increasing $C(h)$ makes these two terms
together at most $h/2$.
Thus every action with present count at least $\theta t$ has
present value at least $V^\star-q-h/2$.
The remaining actions have total weight at most $K\theta$.
Since $p\ge1$, the power mean is at least the weighted arithmetic
mean. Nonnegativity gives
$V_t\ge(1-K\theta)(V^\star-q-h/2)_+
 \ge V^\star-q-h$.
All additional lower bounds on $t$ depend only on $h$ and the
fixed parameters, so they can be included in $C(h)$.
\end{proof}

\begin{proof}[Proof of Theorem~\ref{thm:full_merge_rate}]
We work pathwise on the probability-one event carrying the uniform
noise bounds. Choose $\vartheta\in(\eta,1/2)$ and use the associated $Z$.
All stored values are in $[0,L]$ because of the corrected
normalization and separate evaluator statistics.

\paragraph{A finite accuracy bootstrap.}
Let
\[
    \mathcal B=
    \{(s,a):s\in\mathcal D_H,\ \Delta(s,a)>2d_0\}.
\]
If this set is nonempty, define
\[
    \Gamma=
    \min_{(s,a)\in\mathcal B}
        (\Delta(s,a)-2d_0)>0.
\]
It is positive because the reachable set and action sets are finite.
We prove that all sufficiently large local action counts have
error at most $d_0+\Gamma/4$, with a threshold of at most
polynomial growth in $1+Z$.

Start with $k_0=1$ and $E_0=L$; boundedness supplies this uniform
value-error bound.
Choose $h=(1-\gamma)\Gamma/64$.
Given a uniform value-error bound $E_j$ at all attained counts
at least $k_j$, set
\[
    m_j=
    \left\lceil
       1+\frac{2\gamma LMk_j}{h}
        +\left(\frac{2Z}{h}\right)^{1/\vartheta}
    \right\rceil.
\]
Equation~\eqref{eq:full_merge_coarse_Q} implies that every action
count $m\ge m_j$ has error at most
$\gamma\max(E_j,\varepsilon_H)+h$.
By Lemma~\ref{lem:full_merge_coarse_node}, at all attained node
counts at least
\[
    k_{j+1}
       =\left\lceil C(h)(1+m_j^{1/\rho_-})\right\rceil,
\]
the value error is at most
\[
    E_{j+1}=\gamma\max(E_j,\varepsilon_H)+2h.
\]
These assertions are uniform over states, actual local counts,
and global times.

The map $F(E)=\gamma\max(E,\varepsilon_H)$ is
$\gamma$-Lipschitz and has fixed point $d_0$.
Thus
\[
    E_j\le d_0+\gamma^jL+\frac{2h}{1-\gamma}.
\]
Choose a finite deterministic $J$ such that
$\gamma^J L\le\Gamma/16$.
Then $E_J\le d_0+3\Gamma/32<d_0+\Gamma/8$.
One further use of \eqref{eq:full_merge_coarse_Q}, with threshold
constructed as above, gives action error less than
$d_0+\Gamma/4$ at every attained count beyond a random threshold
$m_\star$.
The finite recursion defining these thresholds gives
\[
    m_\star\le C(1+Z)^r
\]
for finite deterministic constants $C,r$.
There is no rate constant from the final induction in this
construction.

By Lemma~\ref{lem:full_merge_forced_counts}, there is a node-count
threshold $T_\star\le C'(1+Z)^{r'}$ after which all actions at
that node have the stated accuracy at selection time.

\paragraph{Selections of actions with gap above $2d_0$.}
At a selection after this threshold, an action in $\mathcal B$
has estimated value at least $\Gamma/2$ below that of an optimal
action:
\[
    \widehat Q(s,a^\star)-\widehat Q(s,a)
    \ge\Delta(s,a)-2(d_0+\Gamma/4)
    \ge\Gamma/2.
\]
Therefore, if it is selected at depth $h$, its bonus is at least
$\Gamma/2$, implying
\[
    N_a\le
       \left(\frac{2C_h}{\Gamma}\right)^{1/v_h}
       t^{u_h/v_h}
\]
at selection, where $t$ bounds the current completed node count.
The bounded completion delay adds at most $D$ further own updates.
Taking the selection that produced each such action's last completed
update and also accounting
for selections before $T_\star$ gives
\[
    N_a(t)\le T_\star+D+C_\Gamma t^\rho
    \qquad((s,a)\in\mathcal B).
\]
Consequently, their total weight $f_s(t)$ satisfies
\begin{equation}
\label{eq:full_merge_bad_weight}
    L f_s(t)\le R(Z)t^{-(1-\rho)},
\end{equation}
where $R(Z)$ has at most polynomial growth and hence finite
moments of every order.
If $\mathcal B$ is empty, take $f_s(t)=0$ and $R(Z)=0$;
the preceding bootstrap is unnecessary.

Every remaining action has gap at most $d_H$.
For the current weights $w_a=N_a/t$, define
\[
    V_w^\star=
       \left(\sum_a w_a Q^\star(s,a)^p\right)^{1/p}.
\]
Nonnegativity and the weight of the remaining actions give
\[
    V_w^\star\ge
       (1-f_s(t))^{1/p}(V^\star(s)-d_H)_+.
\]
Since $(1-f)^{1/p}\ge1-f$ for $p\ge1$ and $V^\star(s)\le L$,
\begin{equation}
\label{eq:full_merge_target_weight_error}
    |V_w^\star-V^\star(s)|
    \le d_H+R(Z)t^{-(1-\rho)}.
\end{equation}

\paragraph{Chronological induction with a polynomial bound.}
The constant $B_H$ in \eqref{eq:full_merge_finite_bias} satisfies
\begin{equation}
\label{eq:full_merge_bias_fixedpoint}
    B_H=\gamma\max(B_H,\varepsilon_H)+d_H.
\end{equation}
To check this, if $d_H\le(1-\gamma)\varepsilon_H$ then
$B_H=\gamma\varepsilon_H+d_H\le\varepsilon_H$.
Otherwise $B_H=d_H/(1-\gamma)\ge\varepsilon_H$.
Both cases give \eqref{eq:full_merge_bias_fixedpoint}.

Choose the pathwise constant
\[
    A(Z)=
    L+
    \frac{K^\eta Z+R(Z)}{1-g_\eta}.
\]
It has finite expectation.
We prove, in chronological order of completed internal backups,
that every internal value at every attained local count $t\ge1$
satisfies
\begin{equation}
\label{eq:full_merge_pathwise_rate}
    |\widehat V_{s,t}-V^\star(s)|
    \le B_H+A(Z)t^{-\eta}.
\end{equation}
The bound is immediate at an initial local count by the choice
$A(Z)\ge L$.
All nonboundary child returns used by a current update follow
strictly earlier completed updates. Thus the induction hypothesis
can be applied to each such return, even in a cycle.
Distinct global child counts and
\eqref{eq:full_merge_fine_Q} give, at every current positive
action count $m$,
\[
    |\widehat Q_{s,a,m}-Q^\star(s,a)|
    \le
    \gamma\max(B_H,\varepsilon_H)
    +
    \left(
       \frac{\gamma A(Z)M^\eta}{1-\eta}+Z
    \right)m^{-\eta},
\]
where $m^{-\vartheta}\le m^{-\eta}$ has been used.
For previously updated action means the same bound follows from
their earlier histories.

The weighted $\ell_p$ triangle inequality yields
\[
    |\widehat V_{s,t}-V_w^\star|
    \le
    \left(
       \sum_{a:N_a>0}\frac{N_a}{t}
       |\widehat Q_{s,a}-Q^\star(s,a)|^p
    \right)^{1/p}.
\]
Because $p\eta<1$,
\[
    \left(
       \sum_{a:N_a>0}\frac{N_a}{t}N_a^{-p\eta}
    \right)^{1/p}
    \le K^\eta t^{-\eta}.
\]
Therefore
\[
    |\widehat V_{s,t}-V_w^\star|
    \le
    \gamma\max(B_H,\varepsilon_H)
    +(g_\eta A(Z)+K^\eta Z)t^{-\eta}.
\]
Combine this with \eqref{eq:full_merge_target_weight_error},
$\eta<1-\rho$, and \eqref{eq:full_merge_bias_fixedpoint}:
\[
    |\widehat V_{s,t}-V^\star(s)|
    \le
    B_H+
       (g_\eta A(Z)+K^\eta Z+R(Z))t^{-\eta}
    \le B_H+A(Z)t^{-\eta}.
\]
This closes the induction. The constant is allowed to depend on
the whole realized noise envelope because the argument is pathwise;
no conditional concentration claim is made using that constant.

\paragraph{Root budget and expectation.}
After $n$ simulations, the root has at least $n$ completed
internal updates, although cyclic revisits may give it more.
Equation~\eqref{eq:full_merge_pathwise_rate} consequently implies
\[
    |\Vhat_n^H(s_0)-V^\star(s_0)|
    \le B_H+A(Z)n^{-\eta}.
\]
Taking expectations proves \eqref{eq:full_merge_root_rate} with
$C_H=\EE A(Z)<\infty$.

When $\varepsilon_H=0$, one has $d_H=B_H=0$.
If every nonzero gap exceeds $2d_0$, then $d_H=0$ and
$B_H=d_0$.
Moreover, $d_H\le2\gamma\varepsilon_H$ gives
$B_H\le\max\{3\gamma,2\gamma/(1-\gamma)\}\varepsilon_H$.
For completeness, suppose $V_0=V^\star$ at physical terminals
and $\Vtil_H(\cdot,0)=\mathcal B^H V_0$.
The standard truncation bound
$|\Vtil_H(s_0,0)-V^\star(s_0)|\le\gamma^H\varepsilon_H$
and the triangle inequality yield
\begin{equation}
\label{eq:full_merge_finite_truncated_rate}
    \EE|\Vhat_n^H(s_0)-\Vtil_H(s_0,0)|
    \le B_H+\gamma^H\varepsilon_H+C_Hn^{-\eta}.
\end{equation}

For the explicit exponent, $\eta\le1/4$ implies
$-\log(1-\eta)\le2\eta$. Hence
\[
    \log g_\eta
    \le
    \log\gamma+\eta(\log(KM)+2)
    \le\tfrac12\log\gamma<0.
\]
All remaining strict inequalities in
\eqref{eq:full_merge_eta_conditions} follow immediately from
\eqref{eq:full_merge_explicit_eta}.
\end{proof}

\subsection{Proof of Theorem~\ref{thm:adaptive}}

\begin{proof}
For each budget $n$, the horizon $H_n$ is deterministic and fixed
throughout the search. Applying Theorem~\ref{thm:full_merge_rate}
with $H=H_n$ gives \eqref{eq:adaptive_bound_corrected}.
The two stated limits imply convergence in expected absolute error.
Under the additional uniform bounds, $H_n=O(\log n)$ gives
\[
    C_{H_n}n^{-\eta_{H_n}}
    \le C_0(1+H_n)^q n^{-\eta_0}
    =O\!\left((\log n)^q n^{-\eta_0}\right).
\]
\end{proof}
\section{Experimental Details}
\label{sec:experiment_details}

\subsection{Evaluation Protocol}
\label{subsec:eval_protocol}

To evaluate the performance and scalability of each algorithm, we define a set of simulation budgets (number of iterations) tailored to the complexity of each environment. For every environment, each algorithm is evaluated across these simulation checkpoints to observe its performance profile.

Our evaluation follows a replanning paradigm. Specifically, at each time step of an episode, the agent executes a complete search process using its allotted simulation budget to determine the optimal action. This process is repeated for every step until a terminal state is reached. To ensure statistical significance and account for the stochastic nature of Monte Carlo-based methods, we conduct 1,000 independent runs for each combination of algorithm, environment, and simulation budget, except for GBOP, which uses 500 runs. The final performance is reported as the mean cumulative reward (or other environment-specific metric) over the corresponding evaluation runs.

\subsection{Hyperparameter Selection}
\label{subsec:hyperparam_tuning}

To ensure a fair comparison, we perform an extensive hyperparameter sweep for both the proposed method and the baselines using a grid search strategy. The search space for each algorithm is defined based on literature-standard ranges and preliminary runs.

The tuning process is conducted as follows:
\begin{itemize}
    \item For each environment and algorithm, we evaluate candidate hyperparameter sets using $300$ independent runs.
    \item These runs are performed across all predefined simulation budget checkpoints to ensure the robustness of the selected parameters.
    \item \textbf{Selection Criterion:} The optimal hyperparameter configuration is chosen based on the highest aggregate mean reward, calculated by averaging the performance across all simulation budgets within a specific environment.
\end{itemize}

By selecting parameters that perform well across different simulation budgets, we aim to identify configurations that are not only high-performing but also stable under varying computational constraints.
\subsection{Benchmark Environments}
\subsubsection{FrozenLake}
\label{subsec:frozen_grid}

\textbf{Environment Description:} The FrozenLake environment consists of an $8 \times 8$ grid where the agent starts at coordinate $(0,0)$ and aims to reach the goal state at $(7,7)$. The state space is defined by the agent's $(x, y)$ coordinates. The grid cells are categorized into Frozen (F), Holes (H), Start (S), and Goal (G). The action space consists of four discrete movements: \textit{left}, \textit{down}, \textit{right}, and \textit{up}. To introduce stochasticity (slippery mode), transitions are non-deterministic: the executed action results in the intended direction with a probability of $1/3$, or deviates to a perpendicular direction (left turn or right turn) with a probability of $1/3$ each. Boundary collisions result in the agent remaining in its current cell. The episode terminates if the agent reaches the goal, falls into a hole, or exhausts the maximum horizon maximum of $H = 200$. The reward function is strictly sparse, granting a reward of $1$ only upon reaching the goal, and $0$ for all other transitions.

\textbf{Selected Hyperparameters:} UCT($\epsilon=1.5$), MENTS(temperature=0.25, $\epsilon=0.5$), Power-UCT($C=1.5$, $p=2.0$), GS-Power-UCT($C=0.5$, $p = \infty$), GS-Power-UCT-F$^+$($C=1.25$, $p = \infty$).

\subsubsection{Passenger Grid}
\label{subsec:passenger_grid}

\textbf{Environment Description:} The Passenger Grid is formulated on a $7 \times 6$ layout. The state is represented as a tuple $(x, y, M)$, where $M$ is a bitmask tracking the pickup status of passengers. The agent starts at $(0,0)$ and must deliver passengers to the drop-off goal at $(6,0)$. Three passengers are initially located at $(1,2)$, $(0,5)$, and $(6,4)$. The $i$-th bit of $M$ flips to $1$ when the agent occupies passenger $i$'s initial cell. Actions include \textit{left}, \textit{down}, \textit{right}, and \textit{up}. The environment operates in a slippery mode where the effective action is sampled from $\{\text{turn-left}, \text{intended}, \text{turn-right}\}$ with probabilities $\{0.25, 0.5, 0.25\}$, respectively. Boundary collisions keep the agent stationary. Each transition increments the time step, and episodes terminate either upon reaching the drop-off cell or at horizon $H = 70$. The reward is delayed until the goal is reached and depends on the total number of collected passengers: $0$, $1$, $3$, or $7$ for collecting $0$, $1$, $2$, or $3$ passengers, respectively. All other transitions yield a reward of $0$.

\textbf{Selected Hyperparameters:} UCT($\epsilon=1.5$), MENTS(temperature=0.01, $\epsilon=0.5$), Power-UCT($C=1.5$, $p=2.5$), GBOP($\epsilon=0.001$, $\gamma=0.99$), GS-Power-UCT($C=0.5$, $p = \infty$), GS-Power-UCT-F$^+$($C=0.5$, $p=4.0$).

\subsubsection{Factored River Swim}
\label{subsec:riverswim}

\textbf{Environment Description:} The Factored River Swim environment comprises $3$ independent river-swim chains. The joint state captures the position of an agent on each chain, $P \in \{0, \dots, 4\}^3$. All rivers are initialized at position $0$, with position $4$ as the goal. A joint action is represented as a $3$-bit vector ($8$ possible actions), where the $i$-th bit indicates the behavior for river $i$: $0$ for resting/swimming downstream (left) and $1$ for swimming upstream (right). Action $0$ deterministically moves the agent one step left (bounded at $0$). Action $1$ introduces stochastic transitions: at position $0$, the agent stays with probability $0.4$ and moves right with $0.6$; at the goal ($4$), it falls back to $3$ with probability $0.4$ and stays with $0.6$; at any intermediate state, it moves left ($0.05$), stays ($0.60$), or moves right ($0.35$). Episodes terminate at $H = 35$. Rewards are evaluated prior to position updates. For each river, selecting action $0$ at position $0$ yields $0.1$, while selecting action $1$ at position $4$ yields $1.0$. A synergistic bonus of $3.0$ is granted if all rivers are simultaneously at position $4$ and all action bits are set to $1$. The total aggregated reward is normalized by a factor of $6$.

\textbf{Selected Hyperparameters:} UCT($\epsilon=0.5$), MENTS(temperature=0.99, $\epsilon=0.5$), Power-UCT($C=1.5$, $p=1.0$), GBOP($\epsilon=0.001$, $\gamma=0.99$), GS-Power-UCT($C=0.5$, $p=\infty$), GS-Power-UCT-F$^+$($C=1.5$, $p=\infty$).
\subsubsection{SysAdmin Ring}
\label{subsec:sysadmin}

\textbf{Environment Description:} The SysAdmin Ring models a network of $20$ interconnected computers forming a ring topology. The state is represented by an \textit{alive\_mask}, where each bit denotes the running status of a specific machine. Initially, all computers are down (\textit{alive\_mask} $= 0$). The agent can choose from $21$ actions: either reboot a specific computer $i \in \{0, \dots, 19\}$ or remain idle. Rebooting a computer guarantees it will be running in the subsequent state. For all other machines, the status updates stochastically based on their current state and that of their predecessor in the ring. The running probabilities are determined as follows: $0.0238$ if both are down; $0.0475$ if the predecessor is running but the current is down; $0.525$ if the predecessor is down but the current is running; and $0.95$ if both are running. Episodes terminate deterministically at $t = 50$. The environment provides a dense step-wise reward corresponding to the fraction of running machines.

\textbf{Selected Hyperparameters:} UCT($\epsilon=1.5$), MENTS(temperature=0.25, $\epsilon=0.5$), Power-UCT($C=0.5$, $p=3.0$), GS-Power-UCT($C=0.5$, $p=2.5$), GBOP($\epsilon=0.001$, $\gamma=0.99$), GS-Power-UCT-F$^+$($C=0.75$, $p=2.5$).

\subsubsection{FourRooms}
\label{subsec:fourrooms}

\textbf{Environment Description:} FourRooms is an $11 \times 11$ grid partitioned by a vertical wall at $x = 5$ and a horizontal wall at $y = 5$. Communication between rooms is facilitated by four door cells, one randomly placed on each of the four wall segments during reset. The agent's start and goal positions are uniformly sampled from non-wall and non-door cells. The agent executes four directional actions, subject to slippery dynamics identical to the Passenger Grid: effective actions are sampled from $\{\text{turn-left}, \text{intended}, \text{turn-right}\}$ with probabilities $\{0.25, 0.5, 0.25\}$. Collisions with walls or boundaries leave the agent in place. Episodes terminate upon reaching the goal or hitting the time limit $H = 50$. The reward is $0$ at all steps except when the goal is reached, where a time-discounted reward is issued: $R = 1 - 0.9 \times (t / 50)$, incentivizing faster navigation.

\textbf{Selected Hyperparameters:} UCT($\epsilon=1.0$), MENTS(temperature=0.01, $\epsilon=0.5$), Power-UCT($C=1.0$, $p=2.0$), GS-Power-UCT($C=0.5$, $p = \infty$), GBOP($\epsilon=0.001$, $\gamma=0.99$), GS-Power-UCT-F$^+$($C=0.5$, $p=4.0$).

\subsection{Experimental Results}

The following tables report reward as the mean $\pm$ twice the standard deviation over 1,000 evaluation runs for each algorithm and simulation budget (except GBOP: 500 runs). Each table is split into two budget panels for readability; the column headings give the number of simulations per decision. P-UCT denotes Stochastic-Power-UCT.

\begin{table}[!htbp]
\centering
\small
\setlength{\tabcolsep}{4pt}
\renewcommand{\arraystretch}{1.15}
\caption{FrozenLake}
\label{tab:res_frozen}
\begin{tabular*}{\linewidth}{@{}l@{\extracolsep{\fill}}cccc@{}}
\toprule
Algorithm & 16 & 32 & 64 & 128 \\
\midrule
UCT & $0.002{\pm}0.089$ & $0.004{\pm}0.126$ & $0.006{\pm}0.155$ & $0.011{\pm}0.209$ \\
P-UCT & $0.003{\pm}0.109$ & $0.006{\pm}0.155$ & $0.004{\pm}0.126$ & $0.012{\pm}0.218$ \\
MENTS & $0.002{\pm}0.089$ & $0.011{\pm}0.209$ & $0.016{\pm}0.251$ & $0.067{\pm}0.500$ \\
GBOP & $0.000{\pm}0.000$ & $0.000{\pm}0.000$ & $0.060{\pm}0.011$ & $0.198{\pm}0.018$ \\
GS-Power-UCT & $0.051{\pm}0.440$ & $0.140{\pm}0.694$ & $0.227{\pm}0.838$ & $0.314{\pm}0.929$ \\
GS-Power-UCT-F$^+$ & $0.105{\pm}0.613$ & $0.139{\pm}0.692$ & $0.183{\pm}0.774$ & $0.208{\pm}0.812$ \\
\midrule
\addlinespace[0.5em]
Algorithm & 256 & 512 & 1024 & 2048 \\
\midrule
UCT & $0.016{\pm}0.251$ & $0.021{\pm}0.287$ & $0.047{\pm}0.423$ & $0.064{\pm}0.490$ \\
P-UCT & $0.020{\pm}0.280$ & $0.014{\pm}0.235$ & $0.047{\pm}0.423$ & $0.061{\pm}0.479$ \\
MENTS & $0.127{\pm}0.666$ & $0.156{\pm}0.726$ & $0.209{\pm}0.814$ & $0.257{\pm}0.874$ \\
GBOP & $0.236{\pm}0.019$ & $0.262{\pm}0.020$ & $0.246{\pm}0.019$ & $0.294{\pm}0.020$ \\
GS-Power-UCT & $0.395{\pm}0.978$ & $0.480{\pm}1.000$ & $0.527{\pm}0.999$ & $0.560{\pm}0.993$ \\
GS-Power-UCT-F$^+$ & $0.276{\pm}0.894$ & $0.353{\pm}0.956$ & $0.413{\pm}0.985$ & $0.478{\pm}1.000$ \\
\bottomrule
\end{tabular*}
\end{table}

\begin{table}[!htbp]
\centering
\small
\setlength{\tabcolsep}{4pt}
\renewcommand{\arraystretch}{1.15}
\caption{Passenger Grid}
\label{tab:res_passenger}
\begin{tabular*}{\linewidth}{@{}l@{\extracolsep{\fill}}cccc@{}}
\toprule
Algorithm & 16 & 32 & 64 & 128 \\
\midrule
UCT & $0.685{\pm}2.170$ & $0.828{\pm}2.430$ & $1.053{\pm}2.853$ & $1.413{\pm}3.114$ \\
P-UCT & $0.661{\pm}2.216$ & $0.861{\pm}2.598$ & $1.220{\pm}3.143$ & $1.580{\pm}3.249$ \\
MENTS & $0.587{\pm}2.211$ & $0.739{\pm}2.571$ & $0.911{\pm}2.843$ & $1.298{\pm}3.461$ \\
GBOP & $0.000{\pm}0.000$ & $0.000{\pm}0.000$ & $0.000{\pm}0.000$ & $0.046{\pm}0.013$ \\
GS-Power-UCT & $1.877{\pm}4.662$ & $3.062{\pm}5.607$ & $3.919{\pm}5.883$ & $4.544{\pm}5.888$ \\
GS-Power-UCT-F$^+$ & $1.253{\pm}2.552$ & $1.324{\pm}2.516$ & $1.413{\pm}2.810$ & $1.452{\pm}2.831$ \\
\midrule
\addlinespace[0.5em]
Algorithm & 256 & 512 & 1024 & 2048 \\
\midrule
UCT & $1.668{\pm}3.117$ & $2.093{\pm}3.206$ & $2.432{\pm}3.348$ & $2.586{\pm}3.167$ \\
P-UCT & $1.966{\pm}3.587$ & $2.524{\pm}3.815$ & $2.955{\pm}3.976$ & $3.366{\pm}4.138$ \\
MENTS & $1.679{\pm}3.710$ & $2.091{\pm}4.035$ & $2.736{\pm}4.448$ & $3.543{\pm}4.761$ \\
GBOP & $1.570{\pm}0.073$ & $3.756{\pm}0.106$ & $3.452{\pm}0.096$ & $3.546{\pm}0.092$ \\
GS-Power-UCT & $5.178{\pm}5.507$ & $5.363{\pm}5.441$ & $5.801{\pm}4.870$ & $6.134{\pm}4.294$ \\
GS-Power-UCT-F$^+$ & $1.574{\pm}2.951$ & $1.684{\pm}2.957$ & $1.847{\pm}3.170$ & $1.993{\pm}3.364$ \\
\bottomrule
\end{tabular*}
\end{table}

\begin{table}[!htbp]
\centering
\small
\setlength{\tabcolsep}{4pt}
\renewcommand{\arraystretch}{1.15}
\caption{Factored River Swim}
\label{tab:res_riverswim}
\begin{tabular*}{\linewidth}{@{}l@{\extracolsep{\fill}}cccc@{}}
\toprule
Algorithm & 16 & 32 & 64 & 128 \\
\midrule
UCT & $1.319{\pm}0.960$ & $1.344{\pm}0.906$ & $1.597{\pm}1.304$ & $1.809{\pm}1.542$ \\
P-UCT & $1.654{\pm}1.497$ & $1.999{\pm}1.778$ & $2.348{\pm}1.946$ & $2.666{\pm}2.147$ \\
MENTS & $1.121{\pm}0.270$ & $1.352{\pm}0.189$ & $1.535{\pm}0.145$ & $1.645{\pm}0.107$ \\
GBOP & $1.750{\pm}0.000$ & $1.749{\pm}0.000$ & $1.745{\pm}0.006$ & $2.165{\pm}0.032$ \\
GS-Power-UCT & $2.552{\pm}1.434$ & $3.259{\pm}1.784$ & $4.353{\pm}2.166$ & $5.865{\pm}3.522$ \\
GS-Power-UCT-F$^+$ & $5.784{\pm}3.378$ & $6.649{\pm}3.727$ & $6.913{\pm}3.718$ & $6.990{\pm}3.868$ \\
\midrule
\addlinespace[0.5em]
Algorithm & 256 & 512 & 1024 & 2048 \\
\midrule
UCT & $1.995{\pm}1.577$ & $2.286{\pm}1.722$ & $2.440{\pm}1.836$ & $2.437{\pm}1.721$ \\
P-UCT & $3.010{\pm}2.433$ & $3.186{\pm}2.667$ & $2.968{\pm}2.643$ & $2.541{\pm}2.144$ \\
MENTS & $1.720{\pm}0.058$ & $1.749{\pm}0.011$ & $1.750{\pm}0.000$ & $1.750{\pm}0.000$ \\
GBOP & $2.583{\pm}0.043$ & $3.232{\pm}0.059$ & $4.261{\pm}0.076$ & $5.290{\pm}0.073$ \\
GS-Power-UCT & $6.754{\pm}3.664$ & $6.977{\pm}3.586$ & $7.048{\pm}3.690$ & $7.044{\pm}3.687$ \\
GS-Power-UCT-F$^+$ & $6.847{\pm}3.953$ & $6.496{\pm}4.194$ & $5.808{\pm}4.227$ & $5.144{\pm}4.102$ \\
\bottomrule
\end{tabular*}
\end{table}

\begin{table}[!htbp]
\centering
\small
\setlength{\tabcolsep}{4pt}
\renewcommand{\arraystretch}{1.15}
\caption{SysAdmin Ring}
\label{tab:res_sysadmin}
\begin{tabular*}{\linewidth}{@{}l@{\extracolsep{\fill}}cccc@{}}
\toprule
Algorithm & 16 & 32 & 64 & 128 \\
\midrule
UCT & $8.574{\pm}2.692$ & $8.687{\pm}2.835$ & $8.883{\pm}3.039$ & $8.980{\pm}3.105$ \\
P-UCT & $8.575{\pm}2.690$ & $8.731{\pm}2.885$ & $9.078{\pm}3.030$ & $9.479{\pm}3.419$ \\
MENTS & $8.542{\pm}2.734$ & $8.630{\pm}2.894$ & $8.652{\pm}2.856$ & $8.799{\pm}2.966$ \\
GBOP & $9.037{\pm}0.067$ & $8.452{\pm}0.061$ & $8.108{\pm}0.056$ & $8.090{\pm}0.054$ \\
GS-Power-UCT & $8.662{\pm}2.677$ & $9.018{\pm}2.947$ & $9.440{\pm}3.233$ & $10.159{\pm}3.906$ \\
GS-Power-UCT-F$^+$ & $8.867{\pm}2.763$ & $9.334{\pm}3.049$ & $9.757{\pm}3.210$ & $10.656{\pm}3.749$ \\
\midrule
\addlinespace[0.5em]
Algorithm & 256 & 512 & 1024 & 2048 \\
\midrule
UCT & $9.180{\pm}3.372$ & $9.828{\pm}3.986$ & $10.722{\pm}4.249$ & $12.139{\pm}5.583$ \\
P-UCT & $10.197{\pm}4.009$ & $11.168{\pm}4.741$ & $12.872{\pm}5.954$ & $14.810{\pm}6.413$ \\
MENTS & $8.821{\pm}2.985$ & $8.888{\pm}3.116$ & $9.140{\pm}3.346$ & $9.484{\pm}3.514$ \\
GBOP & $8.029{\pm}0.052$ & $8.159{\pm}0.055$ & $8.414{\pm}0.053$ & $8.580{\pm}0.055$ \\
GS-Power-UCT & $11.307{\pm}4.602$ & $13.264{\pm}5.627$ & $15.601{\pm}6.075$ & $17.205{\pm}6.126$ \\
GS-Power-UCT-F$^+$ & $12.218{\pm}4.518$ & $14.107{\pm}5.398$ & $15.892{\pm}5.979$ & $16.794{\pm}6.192$ \\
\bottomrule
\end{tabular*}
\end{table}

\begin{table}[!htbp]
\centering
\small
\setlength{\tabcolsep}{4pt}
\renewcommand{\arraystretch}{1.15}
\caption{FourRooms}
\label{tab:res_fourrooms}
\begin{tabular*}{\linewidth}{@{}l@{\extracolsep{\fill}}cccc@{}}
\toprule
Algorithm & 16 & 32 & 64 & 128 \\
\midrule
UCT & $0.176{\pm}0.618$ & $0.200{\pm}0.628$ & $0.218{\pm}0.660$ & $0.253{\pm}0.672$ \\
P-UCT & $0.165{\pm}0.602$ & $0.208{\pm}0.656$ & $0.245{\pm}0.686$ & $0.281{\pm}0.703$ \\
MENTS & $0.155{\pm}0.591$ & $0.185{\pm}0.627$ & $0.218{\pm}0.655$ & $0.241{\pm}0.668$ \\
GBOP & $0.163{\pm}0.014$ & $0.209{\pm}0.015$ & $0.227{\pm}0.015$ & $0.282{\pm}0.015$ \\
GS-Power-UCT & $0.338{\pm}0.654$ & $0.420{\pm}0.626$ & $0.478{\pm}0.596$ & $0.521{\pm}0.551$ \\
GS-Power-UCT-F$^+$ & $0.396{\pm}0.676$ & $0.443{\pm}0.662$ & $0.469{\pm}0.644$ & $0.484{\pm}0.622$ \\
\midrule
\addlinespace[0.5em]
Algorithm & 256 & 512 & 1024 & 2048 \\
\midrule
UCT & $0.290{\pm}0.689$ & $0.347{\pm}0.706$ & $0.399{\pm}0.694$ & $0.439{\pm}0.675$ \\
P-UCT & $0.328{\pm}0.717$ & $0.381{\pm}0.707$ & $0.407{\pm}0.698$ & $0.457{\pm}0.672$ \\
MENTS & $0.302{\pm}0.685$ & $0.352{\pm}0.694$ & $0.390{\pm}0.688$ & $0.434{\pm}0.671$ \\
GBOP & $0.383{\pm}0.014$ & $0.409{\pm}0.014$ & $0.440{\pm}0.013$ & $0.327{\pm}0.015$ \\
GS-Power-UCT & $0.545{\pm}0.530$ & $0.554{\pm}0.525$ & $0.566{\pm}0.512$ & $0.572{\pm}0.502$ \\
GS-Power-UCT-F$^+$ & $0.494{\pm}0.618$ & $0.501{\pm}0.604$ & $0.509{\pm}0.594$ & $0.513{\pm}0.586$ \\
\bottomrule
\end{tabular*}
\end{table}

\FloatBarrier
\subsection{Runtime and memory analysis}
\label{sec:appendix_overhead}

For a fixed horizon $H$, each search depth requires an expected-$O(1)$ hash lookup and $O(K)$ action-selection and backup work, giving $O(HK)$ time per simulation, the same asymptotic order as tree-based Stochastic-Power-UCT. For depth-augmented GS-Power-UCT, Theorem~\ref{thm:graph_nonworsening} implies that the node and edge counts under the same realised trajectories do not exceed those of the unrolled tree. This representation statement does not apply to independently run planners or to full merging. The graph additionally stores a hash table and adjacency lists, requiring $O(|\cN|+|\cE|)$ auxiliary storage. A finite discrete MDP stores at most $|\cS|(H+1)$ depth-augmented nodes or $|\cS|$ full-merge nodes; before saturation, both quantities are budget-dependent.

To illustrate implementation overhead, we compare the proposed algorithms with Stochastic-Power-UCT. GS-Power-UCT adds state merging. For GS-Power-UCT-F$^+$, each reported budget $n$ uses the fixed horizon $H(n)$ for all $n$ simulations in that search. We report three representative budgets for brevity. Each budget is reported in two rows: average wall-clock time per decision (episode time divided by steps, averaged over 1,000 episodes), and average expanded-node count over all steps.

\textbf{Runtime and memory interpretation:} The largest observed positive runtime difference is 15.5\% (Table~\ref{tab:runtime_analysis}); some measurements are lower than the tree baseline. Expanded-node counts indicate retained graph size and are not direct measurements of RAM use. The proposed methods add hash-table and graph-bookkeeping storage, while node-count reduction can be substantial when transpositions are common. Without transpositions, node counts approach those of the tree representation, with additional bookkeeping constants.

{\small
\setlength{\tabcolsep}{4pt}
\renewcommand{\arraystretch}{1.10}
\setlength{\LTcapwidth}{\linewidth}
\begin{longtable}{@{}p{2.8cm} r l l l l@{}}
\caption{Wall-clock time per decision and average expanded-node count across environments and simulation budgets. Node counts are a retained-graph-size proxy, not memory measurements.}
\label{tab:runtime_analysis}\\
\toprule
\textbf{Environment} & \textbf{Sims} & \textbf{Metric} & \shortstack{\textbf{Stochastic-}\\\textbf{Power-UCT}} & \textbf{GS-Power-UCT} & \textbf{GS-Power-UCT-F$^+$} \\
\midrule
\endfirsthead
\caption[]{Wall-clock time per decision and average expanded-node count (continued).}\\
\toprule
\textbf{Environment} & \textbf{Sims} & \textbf{Metric} & \shortstack{\textbf{Stochastic-}\\\textbf{Power-UCT}} & \textbf{GS-Power-UCT} & \textbf{GS-Power-UCT-F$^+$} \\
\midrule
\endhead
\midrule
\multicolumn{6}{r}{\emph{Continued on next page}} \\
\endfoot
\bottomrule
\endlastfoot
FourRooms & 16 & Time & 0.0146\,ms & 0.0155\,ms (+6.2\%) & 0.0146\,ms (+0.0\%) \\*
 & & Nodes & 17.0 & 8.9 (-47.6\%) & 8.3 (-51.2\%) \\
\cmidrule(lr){2-6}
FourRooms & 256 & Time & 0.2513\,ms & 0.2771\,ms (+10.3\%) & 0.2875\,ms (+14.4\%) \\*
 & & Nodes & 257.0 & 76.9 (-70.1\%) & 51.9 (-79.8\%) \\
\cmidrule(lr){2-6}
FourRooms & 2048 & Time & 4.6116\,ms & 5.3265\,ms (+15.5\%) & 5.2128\,ms (+13.0\%) \\*
 & & Nodes & 1971.7 & 80.8 (-95.9\%) & 54.1 (-97.3\%) \\
\midrule
Passenger Grid & 16 & Time & 0.0199\,ms & 0.0199\,ms (+0.0\%) & 0.0201\,ms (+1.0\%) \\*
 & & Nodes & 17.0 & 9.1 (-46.5\%) & 8.5 (-50.0\%) \\
\cmidrule(lr){2-6}
Passenger Grid & 256 & Time & 0.3155\,ms & 0.3161\,ms (+0.2\%) & 0.3273\,ms (+3.7\%) \\*
 & & Nodes & 257.0 & 75.1 (-70.8\%) & 51.9 (-79.8\%) \\
\cmidrule(lr){2-6}
Passenger Grid & 2048 & Time & 3.9591\,ms & 4.1618\,ms (+5.1\%) & 4.2138\,ms (+6.4\%) \\*
 & & Nodes & 1732.4 & 81.5 (-95.3\%) & 56.8 (-96.7\%) \\
\midrule
Factored River Swim & 16 & Time & 0.0182\,ms & 0.0185\,ms (+1.6\%) & 0.0185\,ms (+1.6\%) \\*
 & & Nodes & 17.0 & 17.0 (-0.0\%) & 17.0 (-0.0\%) \\
\cmidrule(lr){2-6}
Factored River Swim & 256 & Time & 0.3205\,ms & 0.3311\,ms (+3.3\%) & 0.3336\,ms (+4.1\%) \\*
 & & Nodes & 257.0 & 174.5 (-32.1\%) & 161.2 (-37.3\%) \\
\cmidrule(lr){2-6}
Factored River Swim & 2048 & Time & 5.5256\,ms & 5.6520\,ms (+2.3\%) & 5.6782\,ms (+2.8\%) \\*
 & & Nodes & 2049.0 & 702.2 (-65.7\%) & 463.7 (-77.4\%) \\
\midrule
SysAdmin Ring & 16 & Time & 0.0215\,ms & 0.0213\,ms (-0.9\%) & 0.0213\,ms (-0.9\%) \\*
 & & Nodes & 17.0 & 17.0 (-0.0\%) & 17.0 (-0.0\%) \\
\cmidrule(lr){2-6}
SysAdmin Ring & 256 & Time & 0.3730\,ms & 0.3748\,ms (+0.5\%) & 0.3758\,ms (+0.8\%) \\*
 & & Nodes & 257.0 & 241.7 (-6.0\%) & 234.3 (-8.8\%) \\
\cmidrule(lr){2-6}
SysAdmin Ring & 2048 & Time & 6.0448\,ms & 6.1197\,ms (+1.2\%) & 6.1136\,ms (+1.1\%) \\*
 & & Nodes & 2049.0 & 1795.9 (-12.4\%) & 1663.7 (-18.8\%) \\
\midrule
FrozenLake & 16 & Time & 0.1690\,ms & 0.1669\,ms (-1.2\%) & 0.1665\,ms (-1.5\%) \\*
 & & Nodes & 17.0 & 8.7 (-48.8\%) & 8.1 (-52.4\%) \\
\cmidrule(lr){2-6}
FrozenLake & 256 & Time & 2.9186\,ms & 2.9782\,ms (+2.0\%) & 2.5900\,ms (-11.3\%) \\*
 & & Nodes & 257.0 & 74.9 (-70.9\%) & 49.3 (-80.8\%) \\
\cmidrule(lr){2-6}
FrozenLake & 2048 & Time & 30.7731\,ms & 31.5220\,ms (+2.4\%) & 30.0416\,ms (-2.4\%) \\*
 & & Nodes & 2008.7 & 78.0 (-96.1\%) & 51.5 (-97.4\%) \\
\end{longtable}
}

\subsection{Compute resources}
\label{subsec:compute_res}
\begin{table}[htbp]
\centering
\caption{Hardware and software configuration for experiments.}
\label{tab:computational_resources}
\small
\begin{tabular}{ll}
\toprule
\textbf{Component} & \textbf{Specification} \\
\midrule
CPU & Intel Core i9-14900K (24 cores, 32 threads) \\
GPU & NVIDIA GeForce RTX 5090 (32 GB VRAM) \\
System RAM & 62 GB (DDR5) \\
Operating System & Ubuntu 22.04.5 LTS \\
\bottomrule
\end{tabular}
\end{table}

\newpage
\section*{NeurIPS Paper Checklist}

\begin{enumerate}

\item {\bf Claims}
    \item[] Question: Do the main claims made in the abstract and introduction accurately reflect the paper's contributions and scope?
    \item[] Answer: \answerYes{}
    \item[] Justification: The main claims in the abstract and introduction accurately reflect the contributions and scope of the paper, and are consistent with the methods and results presented in the subsequent sections

\item {\bf Limitations}
    \item[] Question: Does the paper discuss the limitations of the work performed by the authors?
    \item[] Answer:\answerNo{} 
    \item[] Justification: A limitations sections is provided in the appendix.

\item {\bf Theory assumptions and proofs}
    \item[] Question: For each theoretical result, does the paper provide the full set of assumptions and a complete (and correct) proof?
    \item[] Answer: \answerYes{}
    \item[] Justification: All theoretical results are fully and rigorously proven in detail in the appendix, with complete derivations and supporting arguments provided

    \item {\bf Experimental result reproducibility}
    \item[] Question: Does the paper fully disclose all the information needed to reproduce the main experimental results of the paper to the extent that it affects the main claims and/or conclusions of the paper (regardless of whether the code and data are provided or not)?
    \item[] Answer: \answerYes{} 
    \item[] Justification: The paper provides the full code along with detailed information on the experimental setup, including environment configuration, hardware specifications, and hyperparameters in appendix and supplementary material, enabling reproduction of the main results.

\item {\bf Open access to data and code}
    \item[] Question: Does the paper provide open access to the data and code, with sufficient instructions to faithfully reproduce the main experimental results, as described in supplemental material?
    \item[] Answer: \answerYes{}
    \item[] Justification: The paper provides open access to the code along with comprehensive instructions in the supplementary material, enabling faithful reproduction of the main experimental results.

\item {\bf Experimental setting/details}
    \item[] Question: Does the paper specify all the training and test details (e.g., data splits, hyperparameters, how they were chosen, type of optimizer) necessary to understand the results?
    \item[] Answer: \answerYes{} 
    \item[] Justification: The paper specifies all relevant training and testing details, including data splits, hyperparameters and their selection process in experimental details section of the appendix.

\item {\bf Experiment statistical significance}
    \item[] Question: Does the paper report error bars suitably and correctly defined or other appropriate information about the statistical significance of the experiments?
    \item[] Answer: \answerYes{}
    \item[] Justification: The paper reports statistical variability using appropriately defined 2-sigma error bars, with detailed descriptions of their computation provided in the experimental details section of the appendix.

\item {\bf Experiments compute resources}
    \item[] Question: For each experiment, does the paper provide sufficient information on the computer resources (type of compute workers, memory, time of execution) needed to reproduce the experiments?
    \item[] Answer:  \answerYes{} 
    \item[] Justification: The paper provides detailed information on the computational resources used for experiment in the experimental details section of the appendix.
    
\item {\bf Code of ethics}
    \item[] Question: Does the research conducted in the paper conform, in every respect, with the NeurIPS Code of Ethics \url{https://neurips.cc/public/EthicsGuidelines}?
    \item[] Answer:\answerYes{}
    \item[] Justification: Yes, the research conducted in this paper conforms in every respect with the NeurIPS Code of Ethics, as it fully adheres to all the stated guidelines regarding research process, potential harms, and impact mitigation measures

\item {\bf Broader impacts}
    \item[] Question: Does the paper discuss both potential positive societal impacts and negative societal impacts of the work performed?
    \item[] Answer: \answerNA{} 
    \item[] Justification: The paper’s topic, methods, and approach have no noticeable or significant societal impact, so the paper did not address them.
    
\item {\bf Safeguards}
    \item[] Question: Does the paper describe safeguards that have been put in place for responsible release of data or models that have a high risk for misuse (e.g., pre-trained language models, image generators, or scraped datasets)?
    \item[] Answer: \answerNA{} 
    \item[] Justification: The paper does not involve the release of data or models with a high risk of misuse; therefore, safeguards for responsible release are not applicable.

\item {\bf Licenses for existing assets}
    \item[] Question: Are the creators or original owners of assets (e.g., code, data, models), used in the paper, properly credited and are the license and terms of use explicitly mentioned and properly respected?
    \item[] Answer: \answerNA{}
    \item[] Justification: The paper does not use any external assets such as third-party code, datasets, or pretrained models; therefore, attribution and licensing considerations are not applicable.

\item {\bf New assets}
    \item[] Question: Are new assets introduced in the paper well documented and is the documentation provided alongside the assets?
    \item[] Answer: \answerNA{} 
    \item[] Justification:The paper does not introduce or release any new assets; therefore, documentation of such assets is not applicable.

\item {\bf Crowdsourcing and research with human subjects}
    \item[] Question: For crowdsourcing experiments and research with human subjects, does the paper include the full text of instructions given to participants and screenshots, if applicable, as well as details about compensation (if any)? 
    \item[] Answer: \answerNA{} 
    \item[] Justification: The paper does not involve crowdsourcing experiments or research with human subjects; therefore, these considerations are not applicable.

\item {\bf Institutional review board (IRB) approvals or equivalent for research with human subjects}
    \item[] Question: Does the paper describe potential risks incurred by study participants, whether such risks were disclosed to the subjects, and whether Institutional Review Board (IRB) approvals (or an equivalent approval/review based on the requirements of your country or institution) were obtained?
    \item[] Answer: \answerNA{} 
    \item[] Justification: The paper does not involve human participants; therefore, considerations regarding participant risk, disclosure, and IRB approval are not applicable.

\item {\bf Declaration of LLM usage}
    \item[] Question: Does the paper describe the usage of LLMs if it is an important, original, or non-standard component of the core methods in this research? Note that if the LLM is used only for writing, editing, or formatting purposes and does \emph{not} impact the core methodology, scientific rigor, or originality of the research, declaration is not required.
    \item[] Answer: \answerNA{}
    \item[] Justification: The LLM was used only for writing and editing purposes and did not contribute to the core methodology.

\end{enumerate}

\end{document}